\documentclass{article} %
\usepackage{arxiv}

\usepackage[utf8]{inputenc} %
\usepackage[T1]{fontenc}    %
\usepackage{hyperref}       %
\usepackage{url}            %
\usepackage{booktabs}       %
\usepackage{amsfonts}       %
\usepackage{nicefrac}       %
\usepackage{microtype}      %
\usepackage{xcolor}         %
\usepackage{natbib}         %

\usepackage{amsmath}
\usepackage{amsthm}
\usepackage{amssymb}
\usepackage{mathtools}

\usepackage{multirow}

\usepackage{stmaryrd}
\usepackage[a]{esvect}
\usepackage{xspace}
\usepackage{dsfont}
\usepackage{bm}
\usepackage{pgfplots}
\usepackage{times}
\usepackage[capitalize,noabbrev]{cleveref}

\usepackage{wrapfig}

\usepackage{tikz}
\usetikzlibrary{calc}
\usetikzlibrary{arrows.meta}
\usetikzlibrary{fit}
\usetikzlibrary{positioning}
\usetikzlibrary{decorations.pathreplacing}
\pgfplotsset{compat=1.18}

\newtheorem{definition}{Assumption}
\newtheorem{lemma}{Lemma}

\newtheorem{theorem}{Theorem}

\newcommand{\ie}{{i.e.}\@\xspace}

\newcommand{\best}[1]{\textbf{#1}}
\newcommand{\rest}[1]{\textit{\textbf{#1}}}

\newcommand{\real}{\ensuremath\mathds{R}}
\newcommand{\binary}{\ensuremath\{0,1\}}

\newcommand{\mat}[1]{{\ensuremath\boldsymbol{\mathbf{#1}}}}

\newcommand{\trace}{\ensuremath\operatorname{tr}} %
\newcommand{\hadamard}{\ensuremath\circ}
\newcommand{\hp}{\ensuremath\hadamard}

\newcommand{\gr}[1]{\ensuremath{#1}}

\newcommand{\D}{\gr{D}}

\newcommand{\adj}{\ensuremath{\mat{A}}}
\newcommand{\hadj}{\ensuremath{\mat{H}}}
\newcommand{\wadj}{\ensuremath{\mat{W}}}

\newcommand{\porient}{\ensuremath{\mat{T}_{\prec_{\prt,\varepsilon}}}}
\newcommand{\sorient}{\ensuremath{\mat{S}_{t,\varepsilon}(\prt)}}
\newcommand{\pa}{\ensuremath{\operatorname{pa}}}  %
\newcommand{\relu}{\ensuremath{\operatorname{ReLU}}}

\newcommand{\loss}{\ensuremath\mathcal{L}}
\newcommand{\model}[1]{\textsc{#1}\xspace}

\newcommand{\notears}{\model{notears}}

\newcommand{\cosmo}{\model{cosmo}}

\newcommand{\nocurl}{\model{nocurl}}
\newcommand{\nocurljoint}{\model{nocurl-u}}
\newcommand{\nobears}{\model{nobears}}

\newcommand{\dagma}{\model{dagma}}
\newcommand{\enco}{\model{enco}}

\newcommand{\tmpi}{\model{tmpi}}

\newcommand{\prt}{\ensuremath{\mat{p}}}
\newcommand{\Prt}{{\ensuremath{\mat{P}}}}

\newcommand{\metric}[1]{\text{#1}\xspace}
\newcommand{\tpr}{\metric{TPR}}

\newcommand{\nhd}{\metric{NHD}}
\newcommand{\fpr}{\metric{FPR}}

\newcommand{\auc}{\metric{AUC}}

\title{Constraint-Free Structure Learning with Smooth Acyclic Orientations}

\author{%
Riccardo Massidda\\
Department of Computer Science\\
Università di Pisa\\
\texttt{riccardo.massidda@phd.unipi.it}
\And
Francesco Landolfi\\
Department of Computer Science\\
Università di Pisa\\
\texttt{francesco.landolfi@phd.unipi.it}
\And
Martina Cinquini\\
Department of Computer Science\\
Università di Pisa\\
\texttt{martina.cinquini@phd.unipi.it}
\And
Davide Bacciu\\
Department of Computer Science\\
Università di Pisa\\
\texttt{davide.bacciu@unipi.it}
}

\begin{document}
\maketitle

\begin{abstract}
The structure learning problem consists of fitting data generated by a Directed Acyclic Graph (DAG) to correctly reconstruct its arcs. In this context, differentiable approaches constrain or regularize the optimization problem using a continuous relaxation of the acyclicity property. The computational cost of evaluating graph acyclicity is cubic on the number of nodes and significantly affects scalability. In this paper we introduce \cosmo, a constraint-free continuous optimization scheme for acyclic structure learning. At the core of our method, we define a differentiable approximation of an orientation matrix parameterized by a single priority vector. Differently from previous work, our parameterization fits a smooth orientation matrix and the resulting acyclic adjacency matrix without evaluating acyclicity at any step. Despite the absence of explicit constraints, we prove that \cosmo always converges to an acyclic solution. In addition to being asymptotically faster, our empirical analysis highlights how \cosmo performance on graph reconstruction compares favorably with competing structure learning methods.
\end{abstract}
\section{Introduction}

Directed Acyclic Graphs (DAGs) are a fundamental tool in several fields to represent probabilistic or causal information about the world~\citep{koller2009probabilistic,pearl2009causality}. A fundamental problem in this context concerns the retrieval of the underlying structure between a set of variables, \ie, the problem of identifying which arcs exist between nodes associated to the variables of interest~\citep{spirtes2000causation}. In recent years, applications of structure learning to causal discovery led to growing interest in tackling the problem using gradient-based methods that optimize a smooth representation of a DAG~\citep{vowels2022d}. For instance, while not suitable for causal discovery \emph{per se}~\citep{reisach2021beware}, acyclic structure learners are fundamental components of most state-of-the-art continuous causal discovery algorithms~\citep{lachapelle_gradient-based_2020,brouillard2020differentiable,lorch_amortized_2022}. A well-established technique, popularized by \notears~\citep{zheng2018notears}, consists of computing the trace of the matrix-exponential of the adjacency matrix, which is differentiable and provably zero if and only if the corresponding graph is acyclic.  However, despite their widespread adoption, \notears-like acyclicity constraints impose a cubic number of operations in the number of nodes per optimization step and substantially prevent scalable and applicable continuous discovery algorithms.

\begin{figure}[t]
    \centering
    \def\distance{12mm}
\def\smalldistance{11mm}
\def\largedistance{22mm}
\def\minidistance{6mm}
\begin{tikzpicture}[node distance={\distance}, main/.style = {draw, circle}] 
  \node[main] (X1) {$X_1$}; 
  \node[main] (X2) [right of=X1,xshift=\distance]{$X_2$}; 
  \node[main] (X3) [below of=X1,xshift=\distance]{$X_3$}; 
  \node[main] (X4) [below of=X3,xshift=-\distance]{$X_4$}; 
  \node[main] (X5) [right of=X4,xshift=\distance]{$X_5$}; 
  \draw[-{Latex}] (X1) to[bend left] (X2);
  \draw[-{Latex}] (X1) to[bend right] (X3);
  \draw[-{Latex}] (X2) -- (X1);
  \draw[-{Latex}] (X3) -- (X1);
  \draw[-{Latex}] (X3) to[bend left] (X5);
  \draw[-{Latex}] (X4) -- (X3);
  \draw[-{Latex}] (X4) -- (X5);
  \draw[-{Latex}] (X5) -- (X2);
  \draw[-{Latex}] (X5) -- (X3);
  \draw[-{Latex}] (X5) to[bend left] (X4);

  \draw [decorate,decoration={brace,amplitude=2mm,mirror,raise=\minidistance}] (X4.west) -- (X5.east) node[midway,yshift=-\smalldistance]{Directed Graph};

  \node[main] (XOR1) [right of=X2,xshift=\minidistance]{$X_1$}; 
  \node[main] (XOR2) [right of=XOR1,xshift=\distance]{$X_2$}; 
  \node[main] (XOR3) [below of=XOR1,xshift=\distance]{$X_3$}; 
  \node[main] (XOR4) [below of=XOR3,xshift=-\distance]{$X_4$}; 
  \node[main] (XOR5) [right of=XOR4,xshift=\distance]{$X_5$}; 
  \draw[-{Latex}] (XOR2) -- (XOR1);
  \draw[-{Latex}] (XOR3) -- (XOR1);
  \draw[-{Latex}] (XOR3) -- (XOR2);
  \draw[-{Latex}] (XOR4) -- (XOR1);
  \draw[-{Latex}] (XOR4) to[bend left] (XOR2);
  \draw[-{Latex}] (XOR4) -- (XOR3);
  \draw[-{Latex}] (XOR4) -- (XOR5);
  \draw[-{Latex}] (XOR5) to[bend right] (XOR1);
  \draw[-{Latex}] (XOR5) -- (XOR2);
  \draw[-{Latex}] (XOR5) -- (XOR3);
  
  \draw[-{Latex}, dashed,color=gray] (XOR1) to[bend left] (XOR2);
  \draw[-{Latex}, dashed,color=gray] (XOR3) to[bend right] (XOR5);

  \node[draw=none] (HADAMARD) at ($(X3)!0.5!(XOR3)$) {\LARGE$\otimes$};

  \node[draw=none] (PRT) [above of=XOR3,yshift=20mm] {$
  \prt = \left[
    4.2\hspace{0.5em}
    4.1\hspace{0.5em}
    2.0\hspace{0.5em}
    {-0.4}\hspace{0.5em}
    1.9
  \right]
  $};
  \draw[->] (PRT) -- ($(XOR3.center) - (0, -5.5em)$);
  \draw[->] (PRT) -- ($(XOR3.center) - (0, -5.5em)$) node [midway, fill=white] {$S_{t\varepsilon}$};

  \draw [decorate,decoration={brace,amplitude=2mm,mirror,raise=\minidistance}] (XOR4.west) -- (XOR5.east) node[midway,yshift=-\smalldistance]{Smooth Acyclic Orientation};

  \node[main] (XDEF1) [right of=XOR2,xshift=\minidistance] {$X_1$}; 
  \node[main] (XDEF2) [right of=XDEF1,xshift=\distance]{$X_2$}; 
  \node[main] (XDEF3) [below of=XDEF1,xshift=\distance]{$X_3$}; 
  \node[main] (XDEF4) [below of=XDEF3,xshift=-\distance]{$X_4$}; 
  \node[main] (XDEF5) [right of=XDEF4,xshift=\distance]{$X_5$}; 
  \draw[-{Latex}] (XDEF2) -- (XDEF1);
  \draw[-{Latex}] (XDEF3) -- (XDEF1);
  \draw[-{Latex}] (XDEF4) -- (XDEF3);
  \draw[-{Latex}] (XDEF4) -- (XDEF5);
  \draw[-{Latex}] (XDEF5) -- (XDEF2);
  \draw[-{Latex}] (XDEF5) -- (XDEF3);
  \draw[-{Latex},dashed,color=gray] (XDEF1) to[bend left] (XDEF2);
  \draw[-{Latex},dashed,color=gray] (XDEF3) to[bend left] (XDEF5);
  
  \node[draw=none] (EQUAL) at ($(XOR3)!0.5!(XDEF3)$) {\LARGE$=$};

  \draw [decorate,decoration={brace,amplitude=2mm,mirror,raise=\minidistance}] (XDEF4.west) -- (XDEF5.east) node[midway,yshift=-\smalldistance]{Directed Acyclic Graph};
\end{tikzpicture}
    \caption{%
    \cosmo frames acyclic structure learning as an unconstrained optimization problem. To optimize an acyclic adjacency matrix (\emph{right}), we propose to learn a directed graph (\emph{left}) and a priority vector on the nodes (\emph{top}). To optimize the priority vector, we introduce a \emph{smooth} acyclic orientation (\emph{center}) where each lower-priority node feeds into each higher-priority node. \textcolor{gray}{Gray dashed} arrows denote arcs with approximately zero weight. We prove that by annealing the temperature, the smooth acycic orientation~$S_{t,\varepsilon}$ converges to a discrete orientation and, consequently, to a DAG.
    }\label{fig:main}
\end{figure}
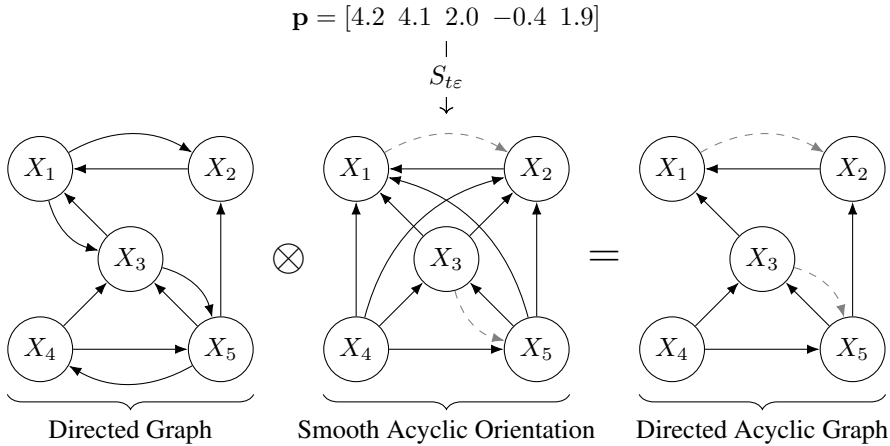

In this context, we propose a novel formulation and optimization scheme for learning acyclic graphs that avoids evaluating the acyclicity of the solution in any optimization step. Notably, our proposal does not sacrifice theoretical guarantees of asymptotic convergence to acyclic solutions which apply to existing structure learning methods~\citep{ng2023structure}. At the core of our scheme lies a novel definition of \emph{smooth} orientation matrix, \ie, a differentiable approximation of an orientation matrix parameterized by a priority vector on the graph nodes. The priority vector represents a discrete orientation where each node has an outgoing arc to all nodes with higher priority. We define our \emph{smooth} orientation matrix by applying a tempered sigmoid to the pair-wise priority differences, which equals the discrete orientation in the limit of the sigmoid temperature to zero. By annealing temperature during training, we prove that the we are effectively decreasing an upper-bound on the acyclicity of the solution. Further, we show that the parameterization represents the space of Directed Acyclic Graph (DAGs) as a differentiable function of a directed graph and our smooth orientation matrix. Since our approach only requires a quadratic number of operations per optimization step, its constraint-free scheme can be used as a direct and faster replacement for the \notears constrained optimization problem.
Overall, we propose a methodology to perform \underline{\textbf{co}}nstraint-free structure learning with \underline{\textbf{sm}}ooth acyclic \underline{\textbf{o}}rientations, which we name \cosmo~(Figure~\ref{fig:main}).

\paragraph{Contributions.}
We summarize the key contributions of this paper as follows:
\begin{itemize}
    \item%
    We introduce a differentiable relaxation of an acyclic orientation matrix, which we name \emph{smooth} orientation matrix~(Definition~\ref{def:smooth_orientation_matrix}). The matrix depends on a temperature value that controls the approximation of the discrete orientation matrix. We prove that we can represent all and only DAGs as the element-wise multiplication of a weighted adjacency matrix and our novel \emph{smooth} orientation matrix~(Theorem~\ref{theo:smooth_orientation}).
    \item%
    We propose \cosmo, the first unconstrained optimization approach that learns a DAG entirely avoiding acyclicity constraints~(Section~\ref{subsec:learnsmoothorient}). \cosmo represents DAGs through a \emph{smooth} orientation matrix and requires solving a unique optimization problem while annealing the temperature. Since reconstructing the DAG requires a number of operations quadratic on the number of nodes, \cosmo is an order of magnitude faster than cubic-expensive constrained methods in literature.
    \item%
    We connect our proposed scheme to existing constrained approaches and prove that annealing the temperature during training effectively decreases an upper bound on the acyclicity of the \emph{smooth} orientation matrix~(Theorem~\ref{theo:smoothub}).
    \item%
    We perform a thorough experimental comparison of \cosmo and competing structure learning approaches~(Section~\ref{sec:experiments}). The empirical results report how \cosmo achieves comparable structure recovery performances in significantly less time. Further, we highlight how \cosmo consistently outperforms previous partially unconstrained structure learning proposals.%
\end{itemize}

In the following, we discuss related works in Section~\ref{sec:related} and report the necessary background on graph theory and structure learning in Section~\ref{sec:background}.
Then, we introduce \cosmo and our original contributions in Section~\ref{sec:method}.
Finally, we report and discuss our empirical analysis in Section~\ref{sec:experiments}.
\section{Related Works}\label{sec:related}

In this section, we report related works aiming to improve, approximate, or avoid altogether the constrained optimization scheme introduced by \notears~\citep{zheng2018notears} for acyclic structure learning.
We summarize the comparison of our proposal with the existing literature in Table~\ref{tab:relatedcomparison}.

\begin{table}
    \caption{%
        Summary comparison of our proposal, \cosmo, with competing approaches. To the best of our knowledge, we are the first to propose a parameterization that enables unconstrained learning of an acyclic graph without trading off on the adjacency matrix rank, the exactness of the acyclicity constraint, or the assumption of observational data. To express computational complexity, we define $d$ as the number of nodes and $k$ as the maximum length of iterative approaches. [$\dagger$]:~\nocurl requires a preliminary solution obtained by partially solving a cubic-expensive constrained optimization problem.
    }
    \label{tab:relatedcomparison}
    \begin{center}
    \begin{tabular}{lccc}
    \toprule
         Method & Complexity & Constraint \\
         \midrule
         \notears~\citep{zheng2018notears}     & $O(d^3)$         & Exact         \\
         \dagma~\citep{bello2022dagma}         & $O(d^3)$         & Exact         \\
         \nobears~\citep{lee2019scaling}       & $O(kd^2)$        & Approximated  \\
         \tmpi~\citep{zhang2022truncated}      & $O(kd^2)$        & Approximated  \\
         \nocurl~\citep{yu2021dags}            & $O{(d^2)}\dagger$ & Partial       \\
         \textbf{\cosmo}                       & $\bm{O(d^2)}$    & \textbf{None} \\
         \bottomrule
    \end{tabular}
    \end{center}
\end{table}

\paragraph{Constraint Reformulation.}
\nobears~\citep{fang2020notearslr} proposes to estimate the acyclicity constraint by approximating the spectral radius of the adjacency matrix. Given a maximum number $k$ of iterations, the constraint can then be evaluated on a graph with $d$ nodes in $O(kd^2)$ time. Similarly, \tmpi~\citep{zhang2022truncated} proposes an iterative approximation of the constraint that also results in $O(kd^2)$ computational complexity. Finally, \dagma~\citep{bello2022dagma} reformulates the acyclicity constraint as the log-determinant of a linear transformation of the adjacency matrix. While still asymptotically cubic in complexity, the use of log-determinant is significantly faster in practice because of widespread optimizations in common linear algebra libraries~\citep[pp.19]{bello2022dagma}.

\paragraph{Low-Rank Approximation.}
Several works extended \notears by assuming that the adjacency matrix of the underlying graph does not have full rank either to reduce the number of trainable parameters~\citep{fang2020notearslr} or to improve computational complexity~\citep{lopez2022dcdfg}. In this work, we deal with possibly full-rank matrices and do not directly compare with low-rank solutions.

\paragraph{Unconstrained DAG Learning.}
Few works learn DAGs without explicitly constraining the graph.
\citet{charpentier2022vidpdag} proposes \model{vp-dag}, an unconstrained and differentiable strategy to sample DAGs that can be easily integrated in probabilistic causal discovery approaches. Instead, we propose a more general optimization scheme for learning acyclic graphs that it is not immediately comparable.
In the context of causal discovery,
\enco~\citep{lippe2022enco} decouples a DAG in an adjacency matrix and an edge orientation matrix. However, the authors explicitly parameterize the orientation matrix and prove that it converges to an acyclic orientation whenever the training dataset contains a sufficient number of interventions. Our structure learning proposal tackles instead non-intervened datasets and ensures acyclicity by construction. 
Similarly to us, \nocurl~\citep{yu2021dags} proposes a model that decouples the topological ordering from the adjacencies of an acyclic graph. However, the proposed optimization schemes are significantly different. Firstly, their approach extracts the nodes ordering from a preliminary solution obtained by partially solving the \notears constrained optimization problem. Then, they fix the ordering and unconstrainedly optimize only the direct adjacency matrix. On the other hand, \cosmo jointly learns priorities and adjacencies avoiding entirely acyclicity evaluations.
We report further discussion on the theoretical comparison with \enco and \nocurl in Appendix~\ref{app:related} and we carefully empirically compare with \nocurl in Section~\ref{sec:experiments}.

\section{Background}\label{sec:background}

\paragraph{Graph Theory.} 
A \emph{directed graph}
is a pair~${\D = (V, A)}$
of vertices~${V = \{1, \dots, d\}}$
and
arcs between them ${A \subseteq V\times V}$.
A directed \emph{acyclic} graph (DAG)
is a directed graph
whose arcs
follow
a strict partial order
on the vertices.
In a DAG,
the \emph{parents}
of a vertex ${v\in V}$
are the set of incoming nodes
such that
${\pa(v) = \{u \in V \mid (u, v) \in A\}}$~\citep{bondy_graph_2008}.
We represent
a directed graph
as a binary adjacency matrix
${\mat{A}\in\binary^{d\times d}}$,
where ${\adj_{uv} \neq 0 \iff (u, v) \in A}$.
Similarly,
we define
a weighted adjacency matrix
as the real matrix
${\mat{W}\in{\real}^{d\times d}}$,
where ${\mat{W}_{uv} \neq 0 \iff (u, v) \in A}$.

\paragraph{Structure Learning.} 
A Structural Equation Model (SEM)
models
a data-generating process
as a set of functions ${f=\{f_1, \dots, f_d\}}$,
where
${f_v: \real^{\lvert \pa(v) \rvert} \to \real}$
for each variable~${v \in V}$ in the DAG~\citep{pearl2009causality}.
Given a class of functions~$\mat{F}$
and a loss~$\loss$,
\notears~\citep{zheng_learning_2020}
formalizes non-linear acyclic structure learning
through the following constrained optimization problem
\begin{align}
    \min_{f \in \mat{F}}
    \loss(f)
    \;\text{ s.t. }\;
    \trace(e^{W(f) \hadamard W(f)}) - d = 0,
\end{align}
where $W(f)\in\real^{d\times d}$
is the adjacency matrix representing
parent relations between variables in $f$.
In particular,
the constraint
equals zero if and only if
the adjacency matrix~$W(f)$
is acyclic.
The authors
propose to solve the problem
using the Augmented Lagrangian
method~\citep{nocedal1999numerical},
which in turn requires to solve
multiple unconstrained problems
and to compute the constraint
value at each optimization step.
Notably,
any causal interpretation
of the identified arcs
depends on several assumptions
on the function class~$\mat{F}$ and
the loss function~$\loss$~\citep{van2013l0,loh2014high},
which we do not explore in this work.

\section{Learning Acyclic Orientations with \cosmo}\label{sec:method}

In Subsection~\ref{subsec:smoothorient},
we propose to parameterize a weighted adjacency matrix
as a function of a direct matrix
and a smooth orientation matrix.
In this way, we effectively
express the discrete space
of DAGs in a continuous and differentiable manner.
Then,
in Subsection~\ref{subsec:learnsmoothorient},
we introduce \cosmo,
our unconstrained optimization approach
to learn acyclic DAGs.
Furthermore,
we prove
an upper bound
on the
acyclicity
of the smooth orientation matrix
that connects
our formulation to constrained approaches.
To ease the presentation,
we initially assume linear relations between variables.
By doing so,
the weighted adjacency matrix is the unique parameter of the problem.
However, as with previous structure learning approaches,
our proposal easily extends
to non-linear models
by jointly optimizing a non-linear model
and an adjacency matrix either weighting or masking variables dependencies.
We report one possible extension
of \cosmo to non-linear relations
in Appendix~\ref{subapp:nonlinear}.

\subsection{Smooth Acyclic Orientations}\label{subsec:smoothorient}

To continuously represent the space of DAGs
with ${d=|V|}$ nodes,
we introduce a priority vector ${\prt\in\real^d}$
on its vertices.
Consequently,
given the priority vector~$\prt$
and
a strictly positive threshold ${\varepsilon>0}$,
we define the following strict partial order
${\prec_{\prt, \varepsilon}}$
on the vertex set $V$
\begin{align}\label{eq:priority-order}
    \forall (u,v) \in V \times V\colon\quad
    u \prec_{\prt,\varepsilon} v
    \iff
    \prt_v - \prt_u \geq \varepsilon.
\end{align}
In other terms,
a vertex $u$ precedes another vertex $v$
if and only if the priority of $v$
is sufficiently larger
than the priority of the vertex $u$.
Notably,
with a zero threshold $\varepsilon=0$,
the relation would be symmetric and thus not a strict order.
On the other hand,
whenever $\varepsilon$ is strictly positive,
we can represent a subset of all strict partial orders
sufficient to express all possible DAGs.
\begin{lemma}\label{lemma:dag-2}
    Let $\mat{W}\in\real^{d\times d}$ be a real matrix.
    Then, for any $\varepsilon>0$,
    $\mat{W}$ is the weighted adjacency matrix of a DAG
    if and only if
    it exists a priority vector $\prt\in\real^d$
    and a real matrix $\mat{H}\in\real^{d\times d}$
    such that
    \begin{align}
        \mat{W} = \mat{H}
        \hp
        \mat{T}_{\prec_{\prt,\varepsilon}},
    \end{align}
    where $\mat{T}_{\prec_{\prt,\varepsilon}}\in\binary^{d\times d}$
    is a binary orientation matrix
    such that
    \begin{align}
        \mat{T}_{\prec_{\prt,\varepsilon}}[uv] =
        \begin{cases}
        1 &\mbox{if $u \prec_{\prt,\varepsilon} v$}\\
        0 &\mbox{otherwise,}
        \end{cases}
    \end{align}
    for any $u,v\in V$.
\end{lemma}
\begin{proof}
    We report the proof in Appendix~\ref{proof:orientationmatrixpriority}.
\end{proof}

While priority vectors enable the representation of strict partial orders in a continuous space, the construction of the orientation matrix still requires the non-differentiable evaluation of the inequality between priority differences from Equation~\ref{eq:priority-order}.
To this end,
we approximate the comparison
of the difference against the threshold $\varepsilon$,
using a \emph{tempered} sigmoidal function.
We refer to such approximation of the orientation matrix as the \emph{smooth} orientation matrix.
\begin{definition}[Smooth Orientation Matrix]\label{def:smooth_orientation_matrix}
    Let $\prt\in\real^d$ be a priority vector,
    $\varepsilon>0$ be a strictly positive threshold,
    and $t>0$ be a strictly positive temperature.
    Then, the \emph{smooth} orientation matrix
    of the strict partial order 
    $\prec_{\prt,\varepsilon}$
    is the real matrix
    $S_{t,\varepsilon}(\prt)\in\real^{d\times d}$
    such that,
    for any $u, v\in V$,
    it holds
    \begin{align}
        S_{t,\varepsilon}(\prt)_{uv}
        = \sigma_{t,\varepsilon}(\prt_v - \prt_u),
    \end{align}
    where $\sigma_{t,\varepsilon}$
    is the $\varepsilon$-centered
    tempered sigmoid,
    defined as
    \begin{align}
        \sigma_{t,\varepsilon}(x) = \frac{1}{1 + e^{-(x-\varepsilon)/t}}.
    \end{align}
\end{definition}

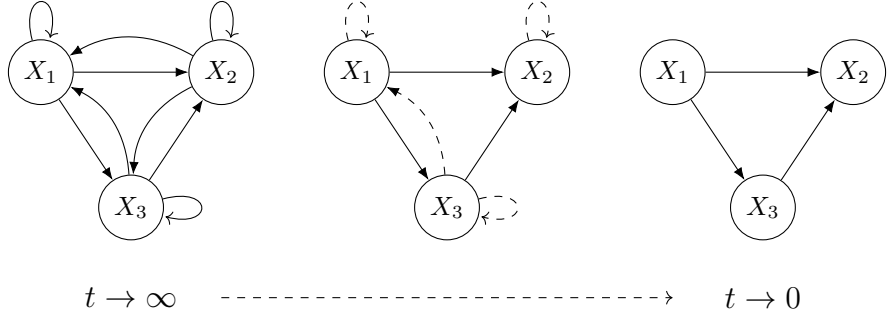
\begin{figure}
    \centering
    \def\distance{12mm}
\def\smalldistance{11mm}
\def\largedistance{22mm}
\def\minidistance{6mm}
\begin{tikzpicture}[node distance={\distance}, main/.style = {draw, circle}] 
  \node[main] (X1) {$X_1$}; 
  \node[main] (X2) [right of=X1,xshift=\distance]{$X_2$}; 
  \node[main] (X3) [below of=X1,xshift=\distance,yshift=-\minidistance]{$X_3$}; 
  \draw[-{Latex}] (X1) -- (X2);
  \draw[-{Latex}] (X1) -- (X3);
  \draw[-{Latex}] (X3) -- (X2);
  \draw[-{Latex}] (X2) to[bend right=30] (X1);
  \draw[-{Latex}] (X2) to[bend right=30] (X3);
  \draw[-{Latex}] (X3) to[bend right=30] (X1);
  \path (X1) edge [loop above] (X1);
  \path (X2) edge [loop above] (X2);
  \path (X3) edge [loop right] (X3);

  \node[main] (XMID1) [right of=X2,xshift=\minidistance]{$X_1$}; 
  \node[main] (XMID2) [right of=XMID1,xshift=\distance]{$X_2$}; 
  \node[main] (XMID3) [below of=XMID1,xshift=\distance,yshift=-\minidistance]{$X_3$}; 
  \draw[-{Latex}] (XMID1) -- (XMID2);
  \draw[-{Latex}] (XMID1) -- (XMID3);
  \draw[-{Latex}] (XMID3) -- (XMID2);
  \draw[-{Latex},dashed] (XMID3) to[bend right=30] (XMID1);
  \path (XMID1) edge [loop above,dashed] (XMID1);
  \path (XMID2) edge [loop above,dashed] (XMID2);
  \path (XMID3) edge [loop right,dashed] (XMID3);

  \node[main] (XDEF1) [right of=XMID2,xshift=\minidistance] {$X_1$}; 
  \node[main] (XDEF2) [right of=XDEF1,xshift=\distance]{$X_2$}; 
  \node[main] (XDEF3) [below of=XDEF1,xshift=\distance,yshift=-\minidistance]{$X_3$}; 
  \draw[-{Latex}] (XDEF1) -- (XDEF2);
  \draw[-{Latex}] (XDEF1) -- (XDEF3);
  \draw[-{Latex}] (XDEF3) -- (XDEF2);

  \node[draw=none] (TINF) [below of=X3] {\large$t \to \infty$};
  \node[draw=none] (TZERO) [below of=XDEF3] {\large$t \to 0$};
  \draw[->,dashed] (TINF) ($(TINF.center) - (-\distance, 0)$) -- ($(TZERO.center) - (\distance, 0)$);

\end{tikzpicture}
    \caption{
        (\emph{left})~With infinite temperature,
        the sigmoid function is constant
        and connects all vertices.
        (\emph{center})~Given two nodes,
        for positive temperatures
        the smooth orientation matrix
        has larger values on the arcs
        respecting the priority ordering.
        (\emph{right})~In the limit of the temperature to zero,
        the smooth orientation matrix contains non-zero entries
        if and only if the arc respects the order,
        \ie,
        it directs a node to another with sufficiently higher priority.
    }\label{fig:annealing}
\end{figure}

Intuitively,
the threshold $\varepsilon$
shifts the center of the sigmoid
and breaks the symmetry
whenever two variables \emph{approximately}
have the same priority.
The temperature parameter $t > 0$
regulates instead
the steepness of the sigmoid.
Because of the asymmetry introduced by the threshold,
in the limit of the temperature to zero,
the zero-entries of a smooth orientation matrix
coincide with the zero-entries of the corresponding orientation matrix~(Figure~\ref{fig:annealing}).
Therefore,
we prove that any directed acyclic graph
can be represented as the element-wise product
of a directed adjacency matrix and a smooth orientation.
Further,
any directed graph resulting
from this decomposition
is acyclic.
\begin{theorem}\label{theo:smooth_orientation}
    Let $\mat{W}\in\real^{d\times d}$ be a real matrix.
    Then, for any $\varepsilon>0$,
    $\mat{W}$ is the weighted adjacency matrix of a DAG
    if and only if
    it exists a priority vector $\prt\in\real^d$
    and a real matrix $\mat{H}\in\real^{d\times d}$
    such that
    \begin{align}
        \mat{W} = \mat{H}
        \hp
        \lim_{t \to 0}
        S_{t,\varepsilon}(\prt),
    \end{align}
    where $S_{t,\varepsilon}(\prt)$ is the
    smooth orientation matrix of $\prec_{\prt,\varepsilon}$.
\end{theorem}
\begin{proof}
    We report the proof in Appendix~\ref{proof:smoothorienteddag}.
\end{proof}

\subsection{Learning Adjacencies and Orientations}\label{subsec:learnsmoothorient}

Given our definition of smooth acyclic orientation,
we can effectively parameterize the space of DAGs
as a continuous function
of a weighted adjacency matrix~$\mat{H}\in\real^{d \times d}$
and a priority vector~$\prt\in\real^d$.
Therefore,
the computational complexity
of our solution reduces
to the construction
of the adjacency matrix~${\mat{W} = \mat{H} \hp S_{t,\varepsilon}(\prt)}$,
which can be achieved in $O(d^2)$ time and space
per optimization step
by computing each arc as
${\mat{W}_{uv} = \mat{H}_{uv} \cdot \sigma((\prt_v-\prt_u-\varepsilon)/t)}$.
In the literature,
\nocurl proposed a slightly similar model
motivated by Hodge Theory~\citep{hodge1989theory}
where each arc is modeled as
${\mat{W}_{uv} = \mat{H}_{uv} \cdot \operatorname{ReLU}(\prt_v-\prt_u)}$.
To avoid a significant performance drop,
their formulation requires a preliminary solution
from a constrained optimization problem
and does not jointly learn
the parameters corresponding
to our adjacencies and priorities.
In the following,
we describe how \cosmo
effectively
reduces to an unconstrained problem
and avoids evaluating acyclicity
altogether.

\paragraph{Temperature Annealing.}
The smooth orientation matrix~$S_{t,\varepsilon}(\prt)$
represents an acyclic orientation
only in the limit
of the temperature~$t\to 0$.
Nonetheless,
the gradient loss
vanishes whenever
the temperature tends to zero.
In fact,
for an arbitrary loss function $\loss$,
we can decompose the gradient
of each component~$\prt_u$ of the priority vector
as follows
\begin{align}\label{eq:chain}
    \frac{\partial \loss(\mat{W})}{\partial \prt_u}
    &=
    \sum_{v\in V}
    \frac{\partial \loss(\mat{W})}
    {\partial \mat{W}_{uv}}
    \cdot
    \frac{\partial \mat{W}_{uv}}{\partial \prt_u} +
    \frac{\partial
    \loss(\mat{W})
    }
    {\partial \mat{W}_{vu}}
    \cdot
    \frac{\partial \mat{W}_{vu}}{\partial \prt_u},\\
    \frac{\partial \mat{W}_{uv}}{\partial \prt_u}
    &=
    -\frac{\mat{H}_{uv}}{t} \sigma_{t,\varepsilon}(\prt_v - \prt_u)
    (1-\sigma_{t,\varepsilon}(\prt_v - \prt_u))\label{eq:prtder_one}\\
    \frac{\partial \mat{W}_{vu}}{\partial \prt_u}
    &=
    \frac{\mat{H}_{vu}}{t} \sigma_{t,\varepsilon}(\prt_v - \prt_u)
    (1-\sigma_{t,\varepsilon}(\prt_v - \prt_u)).\label{eq:prtder_two}
\end{align}
Therefore,
by property of the sigmoidal function~$\sigma_{t,\varepsilon}$
it holds that
both ${\partial \mat{W}_{vu} / \partial \prt_u}$
and ${\partial \mat{W}_{vu} / \partial \prt_u}$
tend to zero for ${t \to 0}$.
To handle this issue,
we tackle
the optimization problem
by progressively reducing
the temperature during training.
In practice,
we perform cosine annealing
from an initial positive value~$t_\text{start}$
to a significantly lower
target value~$t_\text{end}\approx 0$.
We further motivate our choice
by showing the existence
of an upper bound on the acyclicity
of the orientation matrix
that is a monotone increasing function
of the temperature.
Therefore,
temperature annealing
effectively decreases the acyclicity upper bound
during training
of the \emph{smooth} orientation
and, consequently, of the adjacencies.
\begin{theorem}\label{theo:smoothub}
    Let $\prt\in\real^d$ be a priority vector,
    $\varepsilon>0$ be a strictly positive threshold,
    and $t>0$ be a strictly positive temperature.
    Then, given the smooth orientation matrix
    $S_{t,\varepsilon}(\prt)\in\real^{d\times d}$,
    it holds
    \begin{align}
        h(S_{t,\varepsilon}(\prt)) \leq e^{d\alpha} - 1,
    \end{align}
    where
    $h(S_{t,\varepsilon}(\prt))=\trace(e^{S_{t,\varepsilon}(\prt)}) - d$
    is the \notears acyclicity constraint
    and
    $\alpha=\sigma(-\varepsilon/t)$.
\end{theorem}
\begin{proof}
    We report the proof in Appendix~\ref{app:smoothconstraint}.
\end{proof}

\paragraph{Direct Matrix Regularization.}
To contrast the discovery
of spurious arcs
we perform feature selection
by applying L1 regularization
on the adjacency matrix~$\mat{H}$.
Further,
during the annealing procedure,
even if a vertex $u$ precedes $v$
in the partial order~${\prec_{\prt, \varepsilon}}$,
the weight of the opposite arc ${v \to u}$
in the smooth orientation matrix
will only be approximately zero.
Therefore,
sufficiently large values
of the weighted adjacency matrix $\mat{H}$,
might still lead
to undesirable cyclic paths
during training.
To avoid this issue,
we regularize the L2-norm
of the non-oriented adjacency matrix.

\paragraph{Priority Vector Regularization.}
Other than for small temperature values,
the partial derivatives
in Equations~\ref{eq:prtder_one} and \ref{eq:prtder_two}
tend to zero
whenever
the priorities distances~$|\prt_v - \prt_u|$
tend to infinity.
Therefore,
we regularize the L2-norm
of the priority vector.
For the same reason,
we initialize each component
from the normal distribution~${\prt_u\sim\mathcal{N}(0,\varepsilon^2/2)}$,
so that
each difference follows
the normal distribution~${\prt_v - \prt_u\sim\mathcal{N}(0,\varepsilon^2)}$.
We provide further details on initialization in Appendix~\ref{proof:pd}.

\paragraph{Optimization Problem.}
We formalize \cosmo as the differentiable and unconstrained problem
\begin{align}\label{eq:cosmo}
    \min_{\mat{H} \in \real^{d\times d}, \mat{p}\in\real^d}
    \loss(\mat{H} \circ S_{t,\varepsilon}(\prt))
    + \lambda_1\|\mat{H}\|_1
    + \lambda_2\|\mat{H}\|_2
    + \lambda_p\|\mat{p}\|_2,
\end{align}
where $\lambda_1,\lambda_2,\lambda_p$
are the regularization coefficients
for the adjacencies and the priorities.
As the regularization coefficients~$\mat{\lambda}=\{\lambda_1,\lambda_2,\lambda_p\}$,
the initial temperature~$t_\text{start}$,
the target temperature~$t_\text{end}$,
and the shift~$\varepsilon$
are hyperparameters
of our proposal.
Nonetheless,
Theorem~\ref{theo:smoothub}
can guide
the choice
of the final temperature value
and the shift
to guarantee
a maximum tolerance
on the acyclicity
of the smooth orientation matrix.

\section{Experiments}\label{sec:experiments}

We present an experimental comparison of \cosmo against related acyclic structure learning approaches. Our method operates on possibly full-rank graphs and ensures the exact acyclicity of the solution. Therefore, we focus on algorithms providing the same guarantees and under the same conditions. Namely, we confront with the structure learning performance and execution time of \notears~\citep{zheng2018notears}, \nocurl~\citep{yu2021dags}, and \dagma~\citep{bello2022dagma}. As previously discussed, \nocurl proposes a similar model with a substantially different optimization scheme. To highlight the importance of both our parameterization and optimization scheme, we also compare with an entirely unconstrained variant of the algorithm where we directly train the variables ordering without any preliminary solution. In the results, we refer to this variant as \nocurljoint.

We base our empirical analysis on the testbed originally introduced by \citet{zheng2018notears} and then adopted as a benchmark by all followup methods. In particular, we test continuous approaches on randomly generated Erdös-Rényi (ER) and scale-free (SF) graphs of increasing size and for different exogenous noise types.
For each method, we perform structure learning by minimizing the Mean Squared Error (MSE) of a model on a synthetic dataset using the Adam optimizer~\citep{kingma2014adam}.
In Appendix~\ref{app:implementation}, we report further details on the implementation of \cosmo, the baselines, and the datasets.

\subsection{Evaluation Overview}

In line with previous work, we retrieve the binary adjacency matrix by thresholding the learned weights against a small threshold~${\omega=0.3}$~\citep{zheng2018notears}.
While \cosmo guarantees the solution to be acyclic, we maintain the thresholding step to prune correctly oriented but spurious arcs.
Then, we measure the Normalized Hamming Distance~(\nhd) between the solution and the ground-truth as the sum of missing, extra, or incorrect edges divided by the number of nodes. In general, testing weights against a fixed threshold might limit the retrieval of significant arcs with small coefficients in the true model~\citep{xu2022sparse}. For this reason, we also compute the Area under the ROC curve~(\auc), which describes the trade-off between the True Positive Rate~(\tpr) and the False Positive Rate~(\fpr) for increasing values of the weight threshold~\citep{heinze2018causal}. Due to space limitations, we only report in the main body the \auc results, which is the most comprehensive score. Detailed results for other metrics, including \nhd, are provided in the Appendices. 

\subsection{Results Discussion}

\begin{table}[t]
    \caption{%
        Experimental results on linear ER-4 acyclic graphs with different noise terms and sizes. For each algorithm, we report mean and standard deviation over five independent runs of the \auc metric and the time in seconds. We highlight in bold the \best{best} result and in italic bold the \rest{second best} result. The reported duration of \nocurl includes the time to retrieve the necessary preliminary solution using an acyclicity constraint. We denote as \nocurljoint the quadratic version of \nocurl. Complete results on additional metrics and graph types are in Appendix~\ref{app:additionalresults}.
    }
    \label{tab:linearcomparison}
    \begin{center}
        \resizebox{\textwidth}{!}{\begin{tabular}{clrr@{\ $\pm$\ }rrr@{\ $\pm$\ }rrr@{\ $\pm$\ }r}
    \toprule
     &
     &\multicolumn{3}{c}{Gauss}
     &\multicolumn{3}{c}{Exp}
     &\multicolumn{3}{c}{Gumbel}
     \\\cmidrule(lr){3-5}\cmidrule(lr){6-8}\cmidrule(lr){9-11}
     $d$ & Algorithm
     & \multicolumn{1}{c}{\auc} & \multicolumn{2}{c}{Time}
     & \multicolumn{1}{c}{\auc} & \multicolumn{2}{c}{Time}
     & \multicolumn{1}{c}{\auc} & \multicolumn{2}{c}{Time}
     \\
    \midrule
    \multirow{5}{*}{30} & \underline{\cosmo} &
    \rest{0.984 $\pm$ 0.02} & %
    \best{88} & \best{2} & %
    \best{0.989 $\pm$ 0.01} & %
    \best{89} & \best{3} & %
    0.914 $\pm$ 0.10 & %
    \best{87} & \best{2} \\ %
    
    & \dagma &
    \best{0.985 $\pm$ 0.01} & %
    781 & 192 &%
    \rest{0.986 $\pm$ 0.02} & %
    744 & 75 &%
    \rest{0.973 $\pm$ 0.02} & %
    787 & 86 \\ %
    
    & \nocurl &
    0.967 $\pm$ 0.01 & %
    822 & 15 &%
    0.956 $\pm$ 0.02 & %
    826 & 24 &%
    0.915 $\pm$ 0.04 & %
    826 & 17 \\ %
    
    & \nocurljoint &
    0.694 $\pm$ 0.06 & %
    \rest{226} & \rest{5} &%
    0.694 $\pm$ 0.05 & %
    \rest{212} & \rest{5} &%
    0.678 $\pm$ 0.05 & %
    \rest{212} & \rest{5} \\ %
    
    & \notears &
    0.973 $\pm$ 0.02 & %
    5193 & 170 &%
    0.966 $\pm$ 0.03 & %
    5579 & 284 &%
    \best{0.981 $\pm$ 0.01} & %
    5229 & 338 \\ %
    
    \midrule
    \multirow{5}{*}{100} & \underline{\cosmo} &
    0.961 $\pm$ 0.03 & %
    \best{99} & \best{2} & %
    \rest{0.985 $\pm$ 0.01} & %
    \best{99} & \best{2} & %
    \rest{0.973 $\pm$ 0.01} & %
    \best{98} & \best{1} \\ %
    
    & \dagma &
    \best{0.982 $\pm$ 0.01} & %
    660 & 141 & %
    \best{0.986 $\pm$ 0.01} & %
    733 & 109 & %
    \best{0.986 $\pm$ 0.01} & %
    858 & 101 \\ %
    
    & \nocurl &
    0.962 $\pm$ 0.01 & %
    1664 & 14 & %
    0.950 $\pm$ 0.02 & %
    1655 & 28 & %
    0.962 $\pm$ 0.01 & %
    1675 & 34 \\ %
    
    & \nocurljoint &
    0.682 $\pm$ 0.05 & %
    \rest{267} & \rest{10} & %
    0.693 $\pm$ 0.05 & %
    \rest{242} & \rest{4} & %
    0.663 $\pm$ 0.04 & %
    \rest{247} & \rest{9} \\ %
    
    & \notears &
    \rest{0.963 $\pm$ 0.01} & %
    11000 & 339& %
    0.972 $\pm$ 0.01 & %
    10880 & 366 & %
    0.969 $\pm$ 0.00 & %
    11889 & 343 \\ %
    
    \midrule

    \multirow{3}{*}{500} & \underline{\cosmo} &
    \rest{0.933 $\pm$ 0.01}& %
    \best{436}   & \best{81} & %
    \best{0.986 $\pm$ 0.00} & %
    \best{390} & \best{102} & %
    \best{0.982 $\pm$ 0.01} & %
    \best{410} & \best{106} \\ %
    
    & \dagma &
    \best{0.980 $\pm$ 0.00} & %
    2485 & 365 & %
    \rest{0.984 $\pm$ 0.01} & %
    2575 & 469 & %
    \rest{0.980 $\pm$ 0.00} & %
    2853 & 218 \\ %
    
    & \nocurljoint &
    0.683 $\pm$ 0.05& %
    \rest{1546} & \rest{304} & %
    0.715 $\pm$ 0.03 & %
    \rest{1488} & \rest{249} & %
    0.728 $\pm$ 0.05 & %
    \rest{1342} & \rest{209} \\ %
    
    \bottomrule
\end{tabular}
}
    \end{center}
\end{table}

By looking at the \auc of the learned graphs, we observe that \cosmo consistently achieves results that are comparable and competitive with those from constrained-optimization solutions such as \dagma or \notears across different graph sizes and noise types~(Table~\ref{tab:linearcomparison}). This empirically confirms the approximation properties of \cosmo, which can reliably discover DAGs without resorting to explicit acyclicity constraints. 

Furthermore, \cosmo performs better than \nocurl on most datasets. We recall that the latter is the only existing structure learning approach combining constrained and unconstrained optimization. As pointed out in \citet{yu2021dags}, we also observe that the discovery performance of \nocurl drops when optimizing the variable ordering instead of inferring it from a preliminary solution. The fact that \cosmo outperforms \nocurljoint on all datasets highlights
the substantial role and effect of our \emph{smooth} orientation formulation and our optimization scheme for learning the topological
\begin{wrapfigure}{r}{0.45\textwidth}
  \centering
  \resizebox{0.45\textwidth}{!}{\input{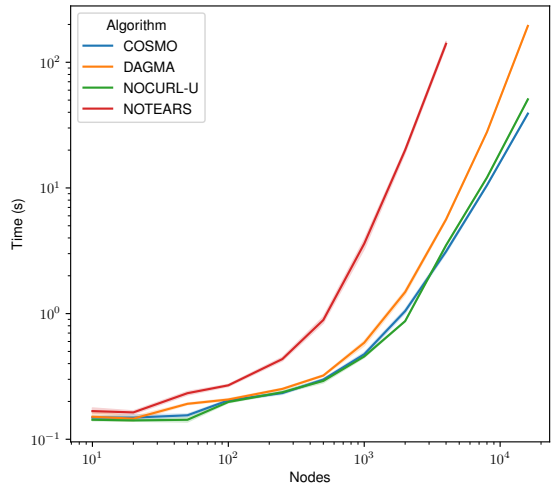}}
  \caption{%
  Average duration of a training epoch for an increasing number of nodes on five independent runs on random ER-4 DAGs.
  }\label{fig:epochtime}
\end{wrapfigure}
ordering of variables from data in an unconstrained way. Overall, our proposal achieves, on average, the best or the second-best result for the \auc metric across all the analyzed datasets and correctly classifies arcs also for large graphs (Figure~\ref{fig:confusion}). Further, as we extensively report in Appendix~\ref{app:additionalresults} for different graph size and noise terms, our non-linear extension obtains comparable performances with \model{dagma-mlp} in significantly less time.

\begin{figure}[t]
    \centering
    \includegraphics[width=\linewidth]{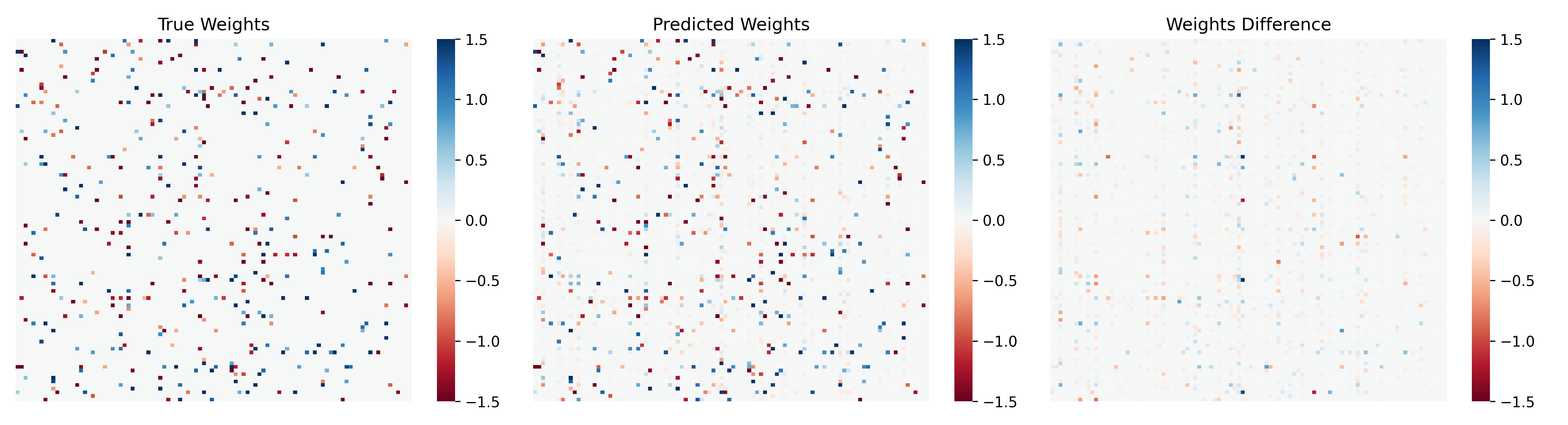}
    \caption{
        Visualization of the weighted adjacency matrix learned by \cosmo (ER4, Gaussian noise, 100 nodes) against the ground truth. We also report the difference between the ground-truth and the learned weights.
        By thresholding the learned weighted adjacency matrix, \cosmo correctly classifies most true (TPR = 0.96) and non-existing arcs (FPR = 0.01), resulting in a limited number of errors (NHD = 0.93) due to the narrow difference in the retrieved weights.
        }\label{fig:confusion}
\end{figure}

Unsurprisingly, due to its quadratic computational complexity, \cosmo is significantly faster than constrained methods on all datasets, especially for increasing graph sizes.
Notably, despite employing early stopping conditions for convergence, all competing methods incur in the cost of solving multiple optimization problems with higher computational cost per step~(Figure~\ref{fig:epochtime}).
In particular, while the unconstrained variant \nocurljoint has a comparable per-epoch average time cost,
for a substantially worse graph recovery performance,
\nocurl overall pays the need for a preliminary solution computed with an acyclicity constraint. Therefore, already for graphs with 500 nodes, only \cosmo, \dagma, and \nocurljoint return a solution before hitting our wall time limit.
Finally, we observe that the cubic computational complexity of \dagma significantly emerges when dealing with large graphs. Therefore, despite the effective underlying optimizations on the log-determinant computation, \dagma's acyclicity constraint still affects scalability.

\begin{table}[t]
    \caption{%
        Ablation test of the priority regularization term $\lambda_p$ on DAGs with different noise terms and sizes. We name the configuration without priority regularization as \model{cosmo-np} and report mean and standard deviation over five independent runs. For each configuration, the \best{best} result is in bold.
    }
    \label{tab:ablation}
    \begin{center}
        \begin{tabular}{cclrrr}
  \toprule
  Graph & $d$ & Algorithm & \nhd & \tpr & \auc\\
  
  \midrule
  
  \multirow{4}{*}{ER4} &
  \multirow{2}{*}{30} & \underline{\textsc{Cosmo}} &
  \best{0.867 $\pm$ 1.01} & %
  \best{0.953 $\pm$ 0.04} & %
  \best{0.984 $\pm$ 0.02} \\ %
  && \textsc{Cosmo-NP} &
  1.893 $\pm$ 0.92 & %
  0.870 $\pm$ 0.07 & %
  0.937 $\pm$ 0.05 \\ %
  
  \cmidrule{2-6}
  
  &
  \multirow{2}{*}{100} & \underline{\textsc{Cosmo}} &
  \best{1.388 $\pm$ 0.69} & %
  0.917 $\pm$ 0.04 & %
  0.961 $\pm$ 0.03 \\ %
  && \textsc{Cosmo-NP} &
  1.570 $\pm$ 0.56 & %
  \best{0.935 $\pm$ 0.04} & %
  \best{0.974 $\pm$ 0.02} \\ %
  
  \cmidrule{1-6}
  
  \multirow{4}{*}{ER6} &
  \multirow{2}{*}{30} & \underline{\textsc{Cosmo}} &
  \best{4.087 $\pm$ 1.12} & %
  \best{0.838 $\pm$ 0.06} & %
  \best{0.921 $\pm$ 0.04} \\ %
  && \textsc{Cosmo-NP} &
  4.153 $\pm$ 2.38 & %
  0.819 $\pm$ 0.12 & %
  0.885 $\pm$ 0.09 \\ %
  
  \cmidrule{2-6}
  
  &
  \multirow{2}{*}{100} & \underline{\textsc{Cosmo}} &
  \best{9.476 $\pm$ 3.01} & %
  0.771 $\pm$ 0.08 & %
  0.911 $\pm$ 0.05 \\ %
  && \textsc{Cosmo-NP} &
  9.804 $\pm$ 2.90 & %
  \best{0.848 $\pm$ 0.05} & %
  \best{0.941 $\pm$ 0.03} \\ %
  
  \bottomrule
\end{tabular}
    \end{center}
\end{table}

Given the proposed parameterization,
\cosmo requires particular care in the choice
of the regularization hyperparameters.
In particular, we carefully considered
the importance of regularizing the priority vector,
which constitutes one of our main differences with previous structure learning approaches.
We found that our hyperpameter validation procedure
consistently returned relatively low priority regularization values
($\lambda_p \approx$ 1e-3).
However,
while it might benefit
structure learning for larger graphs,
ablating priority regularization
results in a non-negligible performance drop
for smaller graphs~(Table~\ref{tab:ablation}).

\section{Conclusion}

We introduced \cosmo, an unconstrained and continuous approach for learning an acyclic graph from data. Our novel definition of \emph{smooth} orientation matrix ensures acyclicity of the solution without requiring the evaluation of computationally expensive constraints. Furthermore, we prove that annealing the temperature of our smooth acyclic orientation corresponds to decreasing an upper bound on the widely adopted acyclicity relaxation from \notears. Overall, our empirical analysis showed that \cosmo performs comparably to constrained methods in significantly less time. Notably, our proposal significantly outperforms the only existing work \emph{partially} optimizing in the space of DAGs, \nocurl, and its completely unconstrained variant \nocurljoint. Therefore, the analysis highlights the role of our parameterization, which does not incur the necessity of preliminary solutions and provably returns a DAG without ever evaluating acyclity.

In recent years, several authors debated using continuous acyclic learners as full-fledged causal discovery algorithms~\citep{reisach2021beware,kaiser2022unsuitability,ng2023structure}. In this context, our empirical analysis of \cosmo shares the same limitations of existing baselines and, exactly like them, might not be significant in the causal discovery scenario. However, acyclic optimization techniques are a fundamental component for continuous discovery approaches~\citep{brouillard2020differentiable,lorch_amortized_2022}. By reducing by an order of magnitude the necessary time to optimize an acyclic causal graph, \cosmo opens up more scalable continuous causal discovery strategies, without sacrificing --- as demonstrated in this work --- the theoretical guarantees on DAGs approximation capabilities.

\bibliography{bibliography}

\newpage
\appendix
\appendix
\section{Deferred Proofs}

\subsection{Proof of Lemma \ref{lemma:dag-2}}\label{proof:orientationmatrixpriority}

\paragraph{\cref{lemma:dag-2}}
\emph{Let $\mat{W}\in\real^{d\times d}$ be a real matrix.
Then, for any $\varepsilon>0$,
$\mat{W}$ is the weighted adjacency matrix of a DAG
if and only if
it exists a priority vector $\prt\in\real^d$
and a real matrix $\mat{H}\in\real^{d\times d}$
such that
\begin{align}
    \mat{W} = \mat{H}
    \hp
    \mat{T}_{\prec_{\prt,\varepsilon}},
\end{align}
where $\mat{T}_{\prec_{\prt,\varepsilon}}\in\binary^{d\times d}$
is a binary orientation matrix
such that
\begin{align}
    \mat{T}_{\prec_{\prt,\varepsilon}}[uv] =
    \begin{cases}
    1 &\mbox{if $u \prec_{\prt,\varepsilon} v$}\\
    0 &\mbox{otherwise,}
    \end{cases}
\end{align}
for any $u,v\in V$.}
    
\begin{proof}
    Firstly,
    we prove
    the existence of a priority vector~$\prt$
    and an adjacency matrix~$\mat{H}$
    for each weighted acyclic matrix~$\mat{W}$
    of a directed acyclic graph~$D=(V, A)$.
    Being a DAG,
    the arcs follow a strict partial order $\prec$
    on the vertices $V=\{1,\dots,d\}$.
    Therefore,
    it holds that
    \begin{align}
        A \subseteq \{(u,v)\mid u \prec v\}.
    \end{align}
    Consequently,
    for an arbitrary topological ordering
    of the variables~$\pi\colon V \to \{1,\dots,d\}$,
    which always exists on DAGs,
    we define the vector~$\prt\in\real^d$
    such that
    \begin{align}
        \prt_u=\varepsilon\pi(u).
    \end{align}
    Given
    the following implications
    \begin{align}
        u \prec v
        \implies
        &\pi(v) > \pi(u)\\
        \implies
        &\prt_v - \prt_u = \varepsilon(\pi(v) - \pi(u)) \geq \varepsilon\\
        \iff
        &u \prec_{\prt,\varepsilon} v,
    \end{align}
    it holds that
    the order $\prec_{\prt,\varepsilon}$
    contains the order $\prec$.
    Finally,
    we can
    define the adjacency matrix
    as $\mat{H}=\mat{W}$,
    where
    $
        \mat{W}=
        \mat{H}
        \hp
        \mat{T}_{\prec_{\prt,\varepsilon}}
    $
    holds since
    $\mat{T}_{\prec_{\prt,\varepsilon}}[u,v]=0$
    only if $(u,v)\not\in A$.

    To prove that
    any priority vector~$\prt$
    and adjacency matrix~$\hadj$
    represent a DAG,
    we first notice
    that,
    since the arcs follow a strict partial order,
    the orientation~$\porient$
    is acyclic.
    Then,
    by element-wise multiplying
    any matrix~$\hadj$
    we obtain a sub-graph
    of a DAG,
    which is acyclic by definition.
\end{proof}

\subsection{Proof of Theorem \ref{theo:smooth_orientation}}\label{proof:smoothorienteddag}

\paragraph{Theorem \ref{theo:smooth_orientation}}
\emph{
Let $\mat{W}\in\real^{d\times d}$ be a real matrix.
Then, for any $\varepsilon>0$,
$\mat{W}$ is the weighted adjacency matrix of a DAG
if and only if
it exists a priority vector $\prt\in\real^d$
and a real matrix $\mat{H}\in\real^{d\times d}$
such that
\begin{align}\label{eq:dagspace_v1}
    \mat{W} = \mat{H}
    \hp
    \lim_{t \to 0}
    S_{t,\varepsilon}(\prt),
\end{align}
where $S_{t,\varepsilon}(\prt)$ is the
smooth orientation matrix of $\prec_{\prt,\varepsilon}$.}
    
\begin{proof}
By \cref{lemma:dag-2},
we know that
for any acyclic weighted adjacency matrix~$\wadj$
there exist a priority vector~$\prt$
and a real matrix~$\hadj$
such that~$\wadj=\hadj\hp\porient$.
Further,
by \cref{def:smooth_orientation_matrix},
the inner limit of Equation \ref{eq:dagspace_v1} solves to  
\begin{align}\label{eq:limit}
    \lim_{t \to 0} S_{t,\varepsilon}(\prt)_{uv}=
\begin{cases}
1 &\prt_v - \prt_u > \varepsilon\\
1/2 &\prt_v - \prt_u = \varepsilon\\
0 &\prt_v - \prt_u < \varepsilon.
\end{cases}
\end{align}
Therefore,
we can define
$\mat{H}'\in\real^{d\times d}$
such that
\begin{align}
\mat{H}_{uv} = 
\begin{cases}
    2\mat{H}'_{uv} & \prt_v - \prt_u = \varepsilon,\\
    \mat{H}'_{uv} & \text{otherwise.}
\end{cases},
\end{align}
from which
\begin{align}
\wadj = \hadj\hp\porient = \hadj'\hp\lim_{t\to 0}\sorient.
\end{align}

Then,
to prove the counter-implication of \cref{theo:smooth_orientation},
we notice that
\begin{align}
    \lim_{t\to 0}\sorient_{uv} = 0
    \iff
    \prt_v - \prt_u < \varepsilon
    \iff
    u \not\prec_{\prt,\varepsilon} v.
\end{align}
Therefore,
since the smooth orientation
contains an arc
if and only if
the vertex respect
the strict partial order~$\prec_{\prt,\varepsilon}$,
it is acyclic.
Consequently,
as in \cref{lemma:dag-2},
the element-wise product
with an acyclic matrix
results in a sub-graph of a DAG,
which is also acyclic by definition.
\end{proof}

\subsection{Priority Vector Initialization}\label{proof:pd}

By independently sampling
each priority component
from a Normal distribution~$\mathcal{N}(\mu, s^2/2)$,
each difference
is consequently sampled from the distribution~$\mathcal{N}(0, s^2)$.
Therefore,
we seek a value for the standard deviation~$s$
that maximizes the partial derivative
\begin{align}
\frac{\partial \mat{W}_{uv}}{\partial \prt_u} = \frac{\mat{H}_{uv}}{t} \sigma_{t,\varepsilon}(\prt_v - \prt_u)
(1-\sigma_{t,\varepsilon}(\prt_v - \prt_u)).
\end{align}
for arbitrary vertices~$u, v$.
Given the definition of the tempered-shifted sigmoid function,
this object
has maximum
in~${\prt_v-\prt_u=\varepsilon}$.
Therefore,
by setting the variance as ${s^2 = \varepsilon^2}$,
we maximize the density function
of the point~${\prt_v-\prt_u=\varepsilon}$
in the distribution~$\mathcal{N}(0, s^2)$.

\section{Smooth Acyclic Orientations and the Acyclicity Constraint}\label{app:smoothconstraint}

In this section,
we present the proof
for the upper bound
on the acyclicity
of a smooth acyclic orientation matrix.
To this end, we introduce two auxiliary and novel lemmas.
Firstly,
we introduce a lemma
which binds the product of a sigmoid
on a sequence of values with zero sum~(Lemma~\ref{lemma:sigmoidprod}).
Then,
we introduce another lemma
on the sum of the priority differences
in a cyclic path~(Lemma~\ref{lemma:cyclediff}).
Finally,
we are able to prove
the acyclicity upper bound
from Theorem~\ref{theo:smoothub}.

\begin{lemma}{(Sigmoid Product Upper Bound)}\label{lemma:sigmoidprod}
    Let $\{x_i\}$ be a sequence of $n$ real numbers such that
    \[
    \sum_{i=1}^n x_i = 0.
    \]
    Then,
    for any temperature $t>0$
    and shift $\varepsilon\geq0$,
    it holds that
    \[
        \prod_{i=1}^n \sigma_{t,\varepsilon}(x_i) \leq \alpha^n,
    \]
    where $\alpha = \sigma(-\varepsilon/t)$
    is the value of the tempered and shifted sigmoid in zero.
\end{lemma}
\begin{proof}
    Before starting,
    we invite the reader to notice that,
    for any temperature~$t>0$,
    if the sum of the sequence~$\{x_i\}$
    is zero,
    then also the sequence~$\{x_i/t\}$
    sums to zero.
    Therefore,
    we omit the temperature in the following proof,
    and assume to divide beforehand all elements of the sequence
    by the temperature~$t$.
    Further,
    we explicitly
    denote the shifted sigmoid
    by using the notation~$\sigma(x_i-\varepsilon)$.

    Firstly,
    we formulate the left-side of the inequality as
    \begin{align*}
        \prod_{i=1}^n \sigma(x_i - \varepsilon)
        =
        \prod_{i=1}^n \frac{e^{x_i-\varepsilon}}{1+e^{x_i-\varepsilon}}
        =
        \frac{\prod_{i=1}^n e^{x_i-\varepsilon}}{\prod_{i=1}^n 1+e^{x_i-\varepsilon}}
        =
        \frac{e^{\sum_{i=1}^n x_i-\varepsilon}}{\prod_{i=1}^n 1+e^{x_i-\varepsilon}}
        =
        \frac{e^{-n\varepsilon}}{\prod_{i=1}^n 1+e^{x_i-\varepsilon}}.
    \end{align*}
    
    Similarly,
    we rewrite the right side as
    \begin{align*} 
        \alpha^n
        =
        \sigma(-\varepsilon)^n
        =
        \left(\frac{e^{-\varepsilon}}{\prod_{i=1}^n 1+e^{-\varepsilon}}\right)^n
        =
        \frac{e^{-n\varepsilon}}{(1+e^{-\varepsilon})^n}.
    \end{align*}

    Therefore,
    proving the left-side smaller or equal
    than the right-side,
    reduces to proving the left-denominator
    is larger than the right-denominator.
    Formally,
    \begin{align*}
        \prod_{i=1}^n 1+e^{x_i-\varepsilon}
        \geq
        \left(1+e^{-\varepsilon}\right)^n,
    \end{align*}
    or equivalently, by applying the logarithmic function,
    \begin{align}\label{ineq:target}
        \sum_{i=1}^n \log(1+e^{x_i-\varepsilon})
        \geq
        n \log(1+e^{-\varepsilon}).
    \end{align}

    To further ease the notation,
    we refer to the left side of inequality~\ref{ineq:target},
    as the target function
    \begin{align*}
        T(x) = \sum_{i=1}^n \log(1+e^{x_i-\varepsilon}).
    \end{align*}
    In particular,
    to prove \ref{ineq:target},
    we show that
    \begin{align}
        \min_{x} T(x)
        =
        n \log(1+e^{-\varepsilon}),
    \end{align}
    for $x = \Vec{0}$,
    which is the only stationary point
    due to the convexity of the target function.

    Without loss of generality,
    we derive the partial derivative
    of the component~$x_1$
    on the target function~$T(x)$.
    To constraint the components sum to zero,
    we consider the components~$\{x_i\}$
    for~$i>2$ as free,
    and then
    $x_2 = -x_1 -\sum_{i=3}^n x_i$
    as a function of the remaining.
    The choice of $x_1, x_2$ is independent from the components ordering,
    and thus applies to any possible pair.
    Consequently,
    \begin{align}
        \frac{\partial T(x)}{\partial x_1}
        &=
        \frac{\partial
        (\log(1+e^{x_1-\varepsilon}) +
        \log(1+e^{-x_1 -\sum_{i=3}^n x_i-\varepsilon}) + 
        \sum_{i=3}^n \log(1+e^{x_i-\varepsilon}))}{
        \partial x_1}\\
        &=
        \frac{\partial
        (\log(1+e^{x_1-\varepsilon}) +
        \log(1+e^{-x_1 -\sum_{i=3}^n x_i-\varepsilon})
        }{
        \partial x_1}\\
        &=
        \sigma(x_1 - \varepsilon)
        -
        \sigma(-x_1 -\sum_{i=3}^n x_i-\varepsilon).
    \end{align}
    Since~$\sigma(-\varepsilon) = \sigma(-\varepsilon)$,
    the equation is satisfied,
    for any component $x_i$,
    by $x=\vec{0}$,

    We finally prove Inequality~\ref{ineq:target},
    by showing that the value of the target function~$T(x)$,
    in its only stationary point $x=\vec{0}$,
    equals the bound.
    Formally,
    \begin{align}
        T(\vec{0})
        &= \sum_{i=1}^n \log(1+e^{-\varepsilon})\\
        &= n \log(1+e^{-\varepsilon}).
    \end{align}
    
\end{proof}

\begin{lemma}{(Sum of Differences in Cycle)}\label{lemma:cyclediff}
    Let $\{p_i\}$ be a sequence
    of $n+1$
    real numbers such that
    ${p_1 = p_{n+1}}$.
    Then, let
    $\{\delta_i\}$
    be a sequence
    of $n$ numbers such that
    $\delta_i = p_{i+1} - p_{i}$.
    Then,
    \begin{align}
        \sum_{i=1}^n \delta_i = 0.
    \end{align}
\end{lemma}
\begin{proof}
    The proof
    is immediate
    from the following sequence
    of equations:
    \begin{align*}
        \sum_{i=1}^n \delta_i
        =
        \sum_{i=1}^n p_{i+1} - p_i
        =
        -p_1 + \sum_{i=2}^n p_{i} - p_i + p_{n+1}
        = 0.
    \end{align*}
\end{proof}

\begin{theorem}{(Orientation Acyclicity Upper Bound)}\label{theo:smoothubapp}
    Let $\mat{W}\in\real^{d\times d}$ be a real matrix.
    Then, for any $\varepsilon>0$,
    $\mat{W}$ is the weighted adjacency matrix of a DAG
    if and only if
    it exists a priority vector $\prt\in\real^d$
    and a real matrix $\mat{H}\in\real^{d\times d}$
    such that
    \begin{align}\label{ineq:upperbound}
        \mat{W} = \mat{H}
        \hp
        \lim_{t \to 0}
        S_{t,\varepsilon}(\prt),
    \end{align}
    where $S_{t,\varepsilon}(\prt)$ is the
    smooth orientation matrix of $\prec_{\prt,\varepsilon}$.
\end{theorem}

\begin{proof}

The left side
of Inequality~\ref{ineq:upperbound}
corresponds
to the following infinite series
\begin{align*}
    \trace(e^\Prt) - d 
    &= \sum_{k=0}^\infty \frac{1}{k!} \trace(\Prt^{(k)}) - d \\
    &= \sum_{k=1}^\infty \frac{1}{k!} \trace(\Prt^{(k)})
\end{align*}
where $\Prt^{(k)}$ is the matrix power defined as $\Prt^{(k)} = \Prt^{(k-1)}\Prt$ and $\Prt^0=\mat{I}$.

By definition of matrix power,
the $u$-th element on the diagonal
of $\Prt^{(k)}$ equals to
\begin{align*}
    \Prt_{uu}^{(k)}
    &=
    \sum_{v_1 \in V} \Prt_{{v_1},u}^{(k-1)}\Prt_{u,{v_1}}\\
    &=
    \sum_{v_1 \in V}
    \cdots
    \sum_{v_{k-1} \in V}
    \Prt_{u,{v_1}}
    \left(\prod_{i=1}^{k-2} \Prt_{v_i,v_{i+1}}\right)
    \Prt_{{v_{k-1}},u}.
\end{align*}
Intuitively,
the $u$-th element on the diagonal
of $\Prt^{(k)}$
amounts to the sum
of all possible paths
starting and ending in the variable $X_u$.
Therefore,
being the same node,
the priority
of the first
and the last node
in the path
are equal by definition.
Consequently,
by Lemma~\ref{lemma:cyclediff},
the difference between the priorities
sums to zero.
For this reason,
given Lemma~\ref{lemma:sigmoidprod},
it holds that the product
of the corresponding sigmoids
is smaller or equal than $\alpha^k$.
Therefore,
\begin{align*}
    \Prt_{uu}^{(k)}
    &=
    \sum_{v_1 \in V}
    \cdots
    \sum_{v_{k-1} \in V}
    \Prt_{u,{v_1}}
    \left(\prod_{i=1}^{k-2} \Prt_{v_i,v_{i+1}}\right)
    \Prt_{{v_{k-1}},u}\\
    &\leq
    \sum_{v_1 \in V}
    \cdots
    \sum_{v_{k-1} \in V}
    \alpha^k\\
    &= d^{k-1}\alpha^k.
\end{align*}
Consequently,
we upper bound
the trace
of the orientation matrix power
as
\begin{align*}
    \trace(\Prt^{(k)})
    = \sum_{u=1}^d \Prt^{(k)}_{uu}
    \leq d^k\alpha^k.
\end{align*}

Finally,
we are able to prove the Theorem as
\begin{align*}
    \trace(e^\Prt) - d 
    &= \sum_{k=0}^\infty \frac{1}{k!} \trace(\Prt^{(k)}) - d \\
    &= \sum_{k=1}^\infty \frac{1}{k!} \trace(\Prt^{(k)}) \\
    &\leq \sum_{k=1}^\infty \frac{1}{k!} d^k\alpha^k \\
    &= -1 + e^{d\alpha},
\end{align*}
where the last passage is due to the Taylor series of the exponential function.
\end{proof}
\section{Implementation Details}\label{app:implementation}

In this section, we discuss the significant aspects of our implementation. We run all the experiments on our internal cluster of Intel(R) Xeon(R) Gold 5120 processors, totaling 56 CPUs per machine. We report details on the evaluation~(\ref{subapp:evaluation}), the data generation procedure~(\ref{subapp:data}), and the models~(\ref{subapp:models}).

\subsection{Evaluation Procedure}\label{subapp:evaluation}

We ensure a fair comparison by selecting the best hyperparameters for each implemented method on each dataset. We describe the hyperparameter space for each algorithm in the following subsections. Firstly, we sample fifty random configurations from the hyperparameter space. Since the hyperparameter space of \cosmo also includes temperature and shift values, we extract more hyperparameters (200 -- 800). Due to the significant speedup of \cosmo, hyperparameter searches take a comparable amount of time, with \notears being significantly longer on small graphs as well. Then, we test each configuration on five randomly sampled DAGs. We select the best hyperparameters according to the average \auc value. Finally, we perform a validation step by running the best configuration on five new random graphs.

To report the duration of each method, we track the time difference from the start of the fitting procedure up to the evaluation process, excluded. For constrained methods, this includes all the necessary adjustments between different problems. For \cosmo, it contains the annealing of the temperature between training epochs.

Following previous work,
we recover the binary adjacency matrix~$\mat{A}$
of the retrieved graph
by thresholding the learned weights~$\mat{W}$
with a small constant $\omega = 0.3$.
Formally, $\mat{A} = |\mat{W}| > \omega$.

\subsection{Synthetic Data}\label{subapp:data}

As we remarked in the main body, continuous approaches are particularly susceptible to data normalization and might exploit variance ordering between variables~\citep{reisach2021beware}. Therefore, empirical results on simulated datasets that do not explicitly control this condition might not generalize to real-world scenarios. Nonetheless, our proposal aims at defining a faster parameterization that could replace existing continuous approaches as a building block in more complex discovery solutions. Therefore, as initially done by \citet{zheng2018notears} and subsequent work, we tested \cosmo and the remaining baselines in the usual synthetic testbed without normalizing the variance.

We include in our code the exact data generation process from the original implementation of \notears.\footnote{%
    \notears implementation
    is published with Apache license
    at \url{https://github.com/xunzheng/notears}.
}
Therefore, the dataset generation procedure firstly produces a DAG with either the Erdős–Rényi (ER) or the scale-free Barabási-Albert (SF) models. Then, it samples 1000 independent observations. In the linear case, the generator simulates equations of the form
\begin{align}
    f_i(x) = \mat{W}^\top_i x + z_i,
\end{align}
where we sample
each weight~$\mat{W}_{ij}$
from the uniform distribution~$\mathcal{U}(-2,-0.5) \cup (0.5, 2)$
and each noise term~$z_i$
from either
the Normal,
Exponential ($\lambda=1$),
or Gumbel ($\mu=0,\beta=1$)
distributions.
In the non-linear case,
we simulate
an additive noise model
with form
\begin{align}
    f_i(x) = g_i(x) + z_i,
\end{align}
where $g_i$ is a randomly initialized
Multilayer Perceptron (MLP)
with 100 hidden units
and the noise term~$z_i$
is sampled from the Normal Distribution~$\mathcal{N}(0,1)$.

\subsection{Non-Linear Relations}\label{subapp:nonlinear}

While linearity
is a common assumption,
the interaction between variables
might be described
by more complex
and possibly non-linear models.
To this end,
we generalize
\cosmo
to represent
non-linear relations
between variables.
To ease the comparison,
we follow the non-linear design
of \model{notears-mlp}~\citep{zheng_learning_2020}
and \model{dagma-mlp}~\citep{bello2022dagma}.
However, it is worth mentioning that this is only one of the possible approaches for non-linear relations and that \cosmo parameterization could be easily extended to mask the input of a neural network instead of the weights, as done by \citet{ng2022masked} or \citet{brouillard2020differentiable}.
Similarly to \notears~\citep{zheng_learning_2020},
we model
the outcome
of each variable $X_u$
with a neural network~${f_u \colon \real^d \to \real}$,
where we distinguish between
the first-layer weights~$\mat{H}^u\in\real^{d\times h}$,
for $h$~distinct neurons,
and the remaining parameters~$\Phi_u$.
We ensure acyclicity
by considering
each weight matrix $\mat{H}^u$
as the $u$-th slice
on the first dimension
of a tensor~$\mat{H}\in\real^{d \times d \times h}$.
Intuitively,
each entry~$\mat{H}^u_{vi}$
represents the weight
from the variable~$X_u$
to the $i$-th neuron
in the first layer
of the MLP~$f_v$.
Then,
we broadcast
the element-wise multiplication
of a smooth orientation matrix
on the hidden dimensions.
Formally,
we model
the structural equation
of each variable~$X_u$
as
\begin{align}
f_u(x) = g_u(\varphi(x^\top \left[\mat{H}\hp S_{t,\varepsilon}(\prt)\right]^u) ; \Phi_u),
\end{align}
where 
$\varphi$ is an activation function
and
$g_u$
is
a Multilayer Perceptron (MLP)
with weights $\Phi_u$.
By applying each MLP~$f_u$
to each variable~$X_u$,
we define the overall SEM
as the function
${f\colon \real^d \to \real^d}$,
which depends
on the parameters~${\mat{\Phi}=\{\Phi_u\}}$,
on the weight tensor~$\mat{H}$,
and the priority vector~$\prt$.
Therefore,
we formalize the non-linear extension of \cosmo
as the following problem
\begin{align}\label{eq:cosmomlp}
    \min_{\mat{H} \in \real^{d\times d}, \mat{p}\in\real^d, \mat{\Phi}}
    \loss(\mat{H} \circ S_{t,\varepsilon}(\prt), \mat{\Phi})
    + \lambda_1\|\mat{H}\|_1
    + \lambda_2\|\mat{H}\|_2
    + \lambda_p\|\mat{p}\|_2.
\end{align}

\subsection{Models}\label{subapp:models}

Since we focus on the role of acyclic learners as a building block within more comprehensive discovery solutions, we slightly detach from experimental setups considering such algorithms as standalone structure learners. Therefore, instead of dealing with full-batch optimization, we perform mini-batch optimization with batch size $B=64$. Similarly, instead of explicitly computing the gradient of the loss function, we implement all methods in PyTorch to exploit automatic differentiation. By avoiding differentiation and other overhead sources, the time expenses results are not directly comparable between our implementations and the results reported in the original papers. However, our implementation choices are common to works that employed \notears \emph{et similia} to ensure the acyclicity of the solution~\citep{lachapelle_gradient-based_2020,brouillard2020differentiable,lopez2022dcdfg}.

By checking the convergence of the model, both \notears, \dagma, and \nocurl can dynamically stop the optimization procedure. On the other hand, \cosmo requires a fixed number of epochs in which to anneal the temperature value. For a fair comparison, while we stop optimization problems after a maximum of 5000 training iterations, we do not disable early-stopping conditions on the baselines. Therefore, when sufficiently large, the maximum number of epochs should not affect the overall execution time of the methods. For \cosmo, we interrupt the optimization after 2000 epochs. For the non-linear version of \dagma, we increased the maximum epochs to 7000. Overall, we interrupt the execution of an algorithm whenever it hits a wall time limit of 20000 seconds.

\begin{table}[t]
    \caption{Hyperparameter ranges and values for \cosmo.}\label{tab:ranges}
    \centering
    \begin{tabular}{lr}
    \toprule
        Hyperparameter & Range/Value \\
        \midrule
         Learning Rate & (1e-3, 1e-2)  \\
         $\lambda_1$ & (1e-4, 1e-3) \\
         $\lambda_2$ & (1e-3, 5e-3) \\
         $\lambda_p$ & (1e-3, 3e-3) \\
         $t_\text{start}$ & 0.45 \\
         $t_\text{end}$ & (5e-4, 1e-3) \\
         $\varepsilon$ & (5e-3, 2e-2)\\
         \bottomrule
    \end{tabular}
\end{table}

As previously discussed in Subsection \ref{subapp:evaluation}, we perform a hyperparameter search on each model for each dataset. In particular, we sampled the learning rate from the range (1e-4, 1e-2) and the regularization coefficients from the interval (1e-4, 1e-1). For the specific constrained optimization parameters, such as the number of problems or decay factors, we replicated the baseline parameters, for which we point the reader to the original papers or our implementation. For \cosmo, we sample hyperparameters from the ranges in Table~\ref{tab:ranges}, given our theoretical findings on the relation between acyclicity and temperature (Theorem~\ref{theo:smoothub}), we ensure sufficiently small acyclicity values. In the non-linear variant, we employ Multilayer Perceptrons with $h=10$ hidden units for each variable.

\section{Detailed Comparison with Related Works}\label{app:related}

\subsection{Comparison with \enco}

\citet{lippe2022enco} propose to learn a directed acyclic graph by jointly learning the probability of an arc being present and of the arc direction. The overall method, named \enco, defines the probability of a direct edge $X_u\to X_v$ as
\begin{align}\label{eq:lippe}
    \wadj_{uv} = \sigma(\mat{H}_{uv}) \cdot \sigma(\mat{P}_{uv}),
\end{align}
where
$\mat{H}\in\real^{d\times d}$
and
$\mat{P}\in\real^{d\times d}$
are free parameters
and $\sigma$ is the sigmoid function.
Although similar to our formulation, \enco defines arc orientations as a matrix that might be cyclic. In fact,
the matrix~$\mat{P}$
does not ensure the transitivity property
that an orientation matrix grants instead.
However,
the authors proved that,
in the limit of the number of samples
from the interventional distribution,
\enco will converge to a directed acyclic graph.
In comparison,
our formulation always
converges to a directed acyclic graph
and is thus adapt
to perform structure learning
in the observational context.

\subsection{Comparison with \nocurl}

\nocurl~\citep{yu2021dags} proposes
to parameterize a DAG
as a function of a $d$-dimensional vector
and a directed possibly cyclic graph.
By adopting our notation,
introduced in Section~\ref{sec:method},
we could report their decomposition as
\begin{align}
    \mat{W}_{uv} = \mat{H}_{uv} \cdot \relu(\prt_v - \prt_u),
\end{align}
for an arbitrary weight from node $X_u$ to node $X_v$.
Compared to our definition,
the use $\relu$ does not correspond
to the approximation of an orientation matrix.
In fact,
the distance between the priorities
directly affects the weight.
Instead,
by employing the shifted-tempered sigmoid
in the definition of smooth acyclic orientation,
in \cosmo%
the priorities only determine
whether an arc is present
between two variables.
Further,
\nocurl requires a preliminary solution
from which to extract the topological
ordering of the variables.
In turn,
such preliminary solution
requires the use
of an acyclicity constraint
in multiple optimization problems.
Therefore,
in practice,
\nocurl does not learn
the variables ordering
in an unconstrained way
and adjusts adjacency weights
in the last optimization problem.

\section{Additional Results}\label{app:additionalresults}

In this section, we report further results on simulated DAGs with different noise terms, graph types, and increasing numbers of nodes. For each algorithm, we present the mean and standard deviation of each metric on five independent runs. We report the Area under the ROC Curve (\auc), the True Positive Ratio (\tpr), and the Structural Hamming Distance normalized by the number of nodes (\nhd). The reported duration of \nocurl includes the time to retrieve the necessary preliminary solution through two optimization problems regularized with the \notears acyclicity constraint. We denote as \nocurljoint the variation of \nocurl that solves a unique unconstrained optimization problem without preliminary solution. When not immediate, we highlight in bold the \best{best} result and in italic bold the \rest{second best} result. We do not report methods exceeding our wall time limit of 20000 seconds.

\subsection{ER4 \-- Gaussian Noise}
\begin{center}
\begin{tabular}{clrrrr}
  \toprule
  $d$ & Algorithm & \nhd & \tpr & \auc & Time (s)\\
  \midrule
  \multirow{5}{*}{30} & \underline{\cosmo} &
  \rest{0.867 $\pm$ 1.01} & %
  \best{0.953 $\pm$ 0.04} & %
  \rest{0.984 $\pm$ 0.02} & %
  \best{88 $\pm$ 3} \\ %
  & \dagma &
  \best{0.707 $\pm$ 0.57} & %
  0.940 $\pm$ 0.04 & %
  \best{0.985 $\pm$ 0.01} & %
  781 $\pm$ 193 \\ %
  & \nocurl &
  1.653 $\pm$ 0.17 & %
  \rest{0.942 $\pm$ 0.02} & %
  0.967 $\pm$ 0.01 & %
  822 $\pm$ 15 \\ %
  & \nocurljoint &
  5.623 $\pm$ 0.92 & %
  0.492 $\pm$ 0.08 & %
  0.694 $\pm$ 0.06 & %
  \rest{227 $\pm$ 5} \\ %
  & \notears &
  0.913 $\pm$ 0.60 & %
  0.940 $\pm$ 0.05 & %
  0.973 $\pm$ 0.02 & %
  5193 $\pm$ 170 \\ %
  \midrule
  \multirow{5}{*}{100} & \underline{\cosmo} &
  \rest{1.388 $\pm$ 0.69} & %
  \rest{0.917 $\pm$ 0.04} & %
  0.961 $\pm$ 0.03 & %
  \best{99 $\pm$ 2} \\ %
  & \dagma &
  \best{1.026 $\pm$ 0.40} & %
  0.876 $\pm$ 0.02 & %
  \best{0.982 $\pm$ 0.01} & %
  661 $\pm$ 142 \\ %
  & \nocurl &
  5.226 $\pm$ 1.34 & %
  \best{0.921 $\pm$ 0.02} & %
  0.962 $\pm$ 0.01 & %
  1664 $\pm$ 15 \\ %
  & \nocurljoint &
  10.108 $\pm$ 4.11 & %
  0.427 $\pm$ 0.05 & %
  0.682 $\pm$ 0.05 & %
  \rest{267 $\pm$ 10} \\ %
  & \notears &
  2.380 $\pm$ 2.10 & %
  0.898 $\pm$ 0.03 & %
  \rest{0.963 $\pm$ 0.01} & %
  11001 $\pm$ 340 \\ %
  \midrule
  \multirow{3}{*}{500} & \underline{\cosmo} &
  \rest{4.149 $\pm$ 1.14} & %
  \rest{0.819 $\pm$ 0.02} & %
  \rest{0.933 $\pm$ 0.01} & %
  \best{437 $\pm$ 81} \\ %
  & \dagma &
  \best{2.246 $\pm$ 0.40} & %
  \best{0.882 $\pm$ 0.01} & %
  \best{0.980 $\pm$ 0.00} & %
  2485 $\pm$ 366 \\ %
  & \nocurljoint &
  27.675 $\pm$ 16.52 & %
  0.410 $\pm$ 0.04 & %
  0.683 $\pm$ 0.05 & %
  \rest{1546 $\pm$ 304} \\ %
  \bottomrule
\end{tabular}
\end{center}

\subsection{ER4 \-- Exponential Noise}
\begin{center}
\begin{tabular}{clrrrr}
  \toprule
  $d$ & Algorithm & \nhd & \tpr & \auc & Time (s)\\
  \midrule
  \multirow{5}{*}{30} & \underline{\cosmo} &
  \best{0.600 $\pm$ 0.54} & %
  \best{0.970 $\pm$ 0.02} & %
  \best{0.989 $\pm$ 0.01} & %
  \best{89 $\pm$ 3} \\ %
  & \dagma &
  \rest{0.613 $\pm$ 0.91} & %
  \rest{0.958 $\pm$ 0.05} & %
  \rest{0.986 $\pm$ 0.02} & %
  744 $\pm$ 75 \\ %
  & \nocurl &
  2.300 $\pm$ 0.97 & %
  0.918 $\pm$ 0.04 & %
  0.956 $\pm$ 0.02 & %
  826 $\pm$ 24 \\ %
  & \nocurljoint &
  5.313 $\pm$ 0.17 & %
  0.423 $\pm$ 0.05 & %
  0.694 $\pm$ 0.05 & %
  \rest{212 $\pm$ 5} \\ %
  & \notears &
  1.320 $\pm$ 0.67 & %
  0.880 $\pm$ 0.10 & %
  0.966 $\pm$ 0.03 & %
  5579 $\pm$ 284 \\ %
  \midrule
  \multirow{5}{*}{100} & \underline{\cosmo} &
  1.642 $\pm$ 0.26 & %
  \best{0.952 $\pm$ 0.02} & %
  \rest{0.985 $\pm$ 0.01} & %
  \best{99 $\pm$ 2} \\ %
  & \dagma &
  \rest{1.294 $\pm$ 0.52} & %
  \rest{0.944 $\pm$ 0.02} & %
  \best{0.986 $\pm$ 0.01} & %
  733 $\pm$ 109 \\ %
  & \nocurl &
  5.652 $\pm$ 1.35 & %
  0.854 $\pm$ 0.03 & %
  0.950 $\pm$ 0.02 & %
  1655 $\pm$ 28 \\ %
  & \nocurljoint &
  11.642 $\pm$ 4.34 & %
  0.478 $\pm$ 0.05 & %
  0.693 $\pm$ 0.05 & %
  \rest{242 $\pm$ 4} \\ %
  & \notears &
  \best{1.156 $\pm$ 0.44} & %
  0.904 $\pm$ 0.03 & %
  0.972 $\pm$ 0.01 & %
  10880 $\pm$ 366 \\ %
  \midrule
  \multirow{3}{*}{500} & \underline{\cosmo} &
  \rest{2.342 $\pm$ 0.86} & %
  \best{0.944 $\pm$ 0.02} & %
  \best{0.986 $\pm$ 0.00} & %
  \best{390 $\pm$ 102} \\ %
  & \dagma &
  \best{2.147 $\pm$ 1.08} & %
  \rest{0.902 $\pm$ 0.04} & %
  \rest{0.984 $\pm$ 0.01} & %
  2575 $\pm$ 469 \\ %
  & \nocurljoint &
  20.183 $\pm$ 7.43 & %
  0.437 $\pm$ 0.03 & %
  0.715 $\pm$ 0.03 & %
  \rest{1488 $\pm$ 249} \\ %
  \bottomrule
\end{tabular}
\end{center}

\subsection{ER4 \-- Gumbel Noise}
\begin{center}
\begin{tabular}{clrrrr}
  \toprule
  $d$ & Algorithm & \nhd & \tpr & \auc & Time (s)\\
  \midrule
  \multirow{5}{*}{30} & \underline{\cosmo} &
  2.220 $\pm$ 1.65 & %
  0.862 $\pm$ 0.14 & %
  0.914 $\pm$ 0.10 & %
  \best{87 $\pm$ 2} \\ %
  & \dagma &
  \rest{1.680 $\pm$ 0.73} & %
  \rest{0.937 $\pm$ 0.03} & %
  \rest{0.973 $\pm$ 0.02} & %
  787 $\pm$ 86 \\ %
  & \nocurl &
  3.873 $\pm$ 1.26 & %
  0.853 $\pm$ 0.08 & %
  0.915 $\pm$ 0.04 & %
  826 $\pm$ 17 \\ %
  & \nocurljoint &
  5.260 $\pm$ 0.57 & %
  0.475 $\pm$ 0.08 & %
  0.678 $\pm$ 0.05 & %
  \rest{212 $\pm$ 5} \\ %
  & \notears &
  \best{0.587 $\pm$ 0.38} & %
  \best{0.962 $\pm$ 0.03} & %
  \best{0.981 $\pm$ 0.01} & %
  5229 $\pm$ 338 \\ %
  \midrule
  \multirow{5}{*}{100} & \underline{\cosmo} &
  2.398 $\pm$ 0.70 & %
  \best{0.936 $\pm$ 0.02} & %
  \rest{0.973 $\pm$ 0.01} & %
  \best{98 $\pm$ 1} \\ %
  & \dagma &
  \best{1.132 $\pm$ 0.79} & %
  \rest{0.921 $\pm$ 0.04} & %
  \best{0.986 $\pm$ 0.01} & %
  858 $\pm$ 101 \\ %
  & \nocurl &
  4.714 $\pm$ 1.77 & %
  0.905 $\pm$ 0.03 & %
  0.962 $\pm$ 0.01 & %
  1675 $\pm$ 34 \\ %
  & \nocurljoint &
  6.914 $\pm$ 0.80 & %
  0.383 $\pm$ 0.04 & %
  0.663 $\pm$ 0.04 & %
  \rest{247 $\pm$ 9} \\ %
  & \notears &
  \rest{1.402 $\pm$ 0.40} & %
  0.869 $\pm$ 0.04 & %
  0.969 $\pm$ 0.00 & %
  11889 $\pm$ 343 \\ %
  \midrule
  \multirow{3}{*}{500} & \underline{\cosmo} &
  \rest{3.574 $\pm$ 1.44} & %
  \best{0.932 $\pm$ 0.02} & %
  \best{0.982 $\pm$ 0.01} & %
  \best{410 $\pm$ 106} \\ %
  & \dagma &
  \best{1.737 $\pm$ 0.64} & %
  \rest{0.871 $\pm$ 0.03} & %
  \rest{0.980 $\pm$ 0.00} & %
  2853 $\pm$ 218 \\ %
  & \nocurljoint &
  18.182 $\pm$ 9.28 & %
  0.462 $\pm$ 0.06 & %
  0.728 $\pm$ 0.05 & %
  \rest{1342 $\pm$ 209} \\ %
  \bottomrule
\end{tabular}
\end{center}

\subsection{SF4 \-- Gaussian Noise}
\begin{center}
\begin{tabular}{clrrrr}
  \toprule
  $d$ & Algorithm & \nhd & \tpr & \auc & Time (s)\\
  \midrule
  \multirow{5}{*}{30} & \underline{\cosmo} &
  \best{0.300 $\pm$ 0.09} & %
  \rest{0.973 $\pm$ 0.01} & %
  \best{0.997 $\pm$ 0.00} & %
  \best{89 $\pm$ 5} \\ %
  & \dagma &
  \rest{0.360 $\pm$ 0.30} & %
  \best{0.973 $\pm$ 0.02} & %
  \rest{0.996 $\pm$ 0.01} & %
  653 $\pm$ 198 \\ %
  & \nocurl &
  0.967 $\pm$ 0.43 & %
  0.893 $\pm$ 0.03 & %
  0.983 $\pm$ 0.01 & %
  828 $\pm$ 23 \\ %
  & \nocurljoint &
  4.410 $\pm$ 0.72 & %
  0.566 $\pm$ 0.11 & %
  0.741 $\pm$ 0.08 & %
  \rest{226 $\pm$ 7} \\ %
  & \notears &
  0.553 $\pm$ 0.54 & %
  0.944 $\pm$ 0.06 & %
  0.984 $\pm$ 0.02 & %
  5292 $\pm$ 261 \\ %
  \midrule
  \multirow{5}{*}{100} & \underline{\cosmo} &
  \rest{0.482 $\pm$ 0.31} & %
  \rest{0.962 $\pm$ 0.02} & %
  0.991 $\pm$ 0.01 & %
  \best{99 $\pm$ 3} \\ %
  & \dagma &
  0.712 $\pm$ 0.33 & %
  0.951 $\pm$ 0.02 & %
  \best{0.995 $\pm$ 0.00} & %
  479 $\pm$ 75 \\ %
  & \nocurl &
  2.030 $\pm$ 0.46 & %
  0.883 $\pm$ 0.03 & %
  0.982 $\pm$ 0.01 & %
  1667 $\pm$ 25 \\ %
  & \nocurljoint &
  5.521 $\pm$ 0.61 & %
  0.596 $\pm$ 0.09 & %
  0.788 $\pm$ 0.06 & %
  \rest{269 $\pm$ 9} \\ %
  & \notears &
  \best{0.280 $\pm$ 0.35} & %
  \best{0.972 $\pm$ 0.04} & %
  \rest{0.993 $\pm$ 0.01} & %
  10112 $\pm$ 492 \\ %
  \midrule
  \multirow{3}{*}{500} & \underline{\cosmo} &
  \rest{1.566 $\pm$ 0.68} & %
  \best{0.953 $\pm$ 0.02} & %
  \rest{0.989 $\pm$ 0.01} & %
  \best{541 $\pm$ 15} \\ %
  & \dagma &
  \best{1.343 $\pm$ 0.46} & %
  \rest{0.915 $\pm$ 0.04} & %
  \best{0.992 $\pm$ 0.00} & %
  \rest{1345 $\pm$ 33} \\ %
  & \nocurljoint &
  7.146 $\pm$ 3.19 & %
  0.504 $\pm$ 0.08 & %
  0.780 $\pm$ 0.07 & %
  1394 $\pm$ 217 \\ %
  \bottomrule
\end{tabular}
\end{center}

\subsection{SF4 \-- Exponential Noise}
\begin{center}
\begin{tabular}{clrrrr}
  \toprule
  $d$ & Algorithm & \nhd & \tpr & \auc & Time (s)\\
  \midrule
  \multirow{5}{*}{30} & \underline{\cosmo} &
  0.613 $\pm$ 0.39 & %
  \rest{0.965 $\pm$ 0.02} & %
  0.985 $\pm$ 0.02 & %
  \best{87 $\pm$ 2} \\ %
  & \dagma &
  \best{0.127 $\pm$ 0.20} & %
  \best{0.991 $\pm$ 0.01} & %
  \best{0.999 $\pm$ 0.00} & %
  592 $\pm$ 200 \\ %
  & \nocurl &
  0.887 $\pm$ 0.21 & %
  0.845 $\pm$ 0.02 & %
  \rest{0.985 $\pm$ 0.01} & %
  824 $\pm$ 25 \\ %
  & \nocurljoint &
  4.067 $\pm$ 0.73 & %
  0.460 $\pm$ 0.15 & %
  0.685 $\pm$ 0.09 & %
  \rest{212 $\pm$ 7} \\ %
  & \notears &
  \rest{0.513 $\pm$ 0.30} & %
  0.962 $\pm$ 0.03 & %
  0.984 $\pm$ 0.01 & %
  5189 $\pm$ 271 \\ %
  \midrule
  \multirow{5}{*}{100} & \underline{\cosmo} &
  \rest{0.724 $\pm$ 0.71} & %
  \rest{0.963 $\pm$ 0.04} & %
  0.985 $\pm$ 0.02 & %
  \best{100 $\pm$ 2} \\ %
  & \dagma &
  \best{0.586 $\pm$ 0.56} & %
  \best{0.969 $\pm$ 0.03} & %
  \best{0.995 $\pm$ 0.00} & %
  395 $\pm$ 108 \\ %
  & \nocurl &
  1.998 $\pm$ 0.40 & %
  0.907 $\pm$ 0.03 & %
  0.980 $\pm$ 0.00 & %
  1670 $\pm$ 28 \\ %
  & \nocurljoint &
  5.912 $\pm$ 1.54 & %
  0.575 $\pm$ 0.06 & %
  0.783 $\pm$ 0.04 & %
  \rest{245 $\pm$ 7} \\ %
  & \notears &
  0.910 $\pm$ 0.43 & %
  0.962 $\pm$ 0.02 & %
  \rest{0.991 $\pm$ 0.01} & %
  10243 $\pm$ 723 \\ %
  \midrule
  \multirow{3}{*}{500} & \underline{\cosmo} &
  \best{1.445 $\pm$ 0.58} & %
  \best{0.950 $\pm$ 0.03} & %
  \best{0.990 $\pm$ 0.01} & %
  \best{517 $\pm$ 108} \\ %
  & \dagma &
  \rest{1.653 $\pm$ 0.91} & %
  \rest{0.873 $\pm$ 0.08} & %
  \rest{0.988 $\pm$ 0.01} & %
  1466 $\pm$ 247 \\ %
  & \nocurljoint &
  12.140 $\pm$ 7.84 & %
  0.482 $\pm$ 0.08 & %
  0.727 $\pm$ 0.06 & %
  \rest{1205 $\pm$ 257} \\ %
  \bottomrule
\end{tabular}
\end{center}

\subsection{SF4 \-- Gumbel Noise}
\begin{center}
\begin{tabular}{clrrrr}
  \toprule
  $d$ & Algorithm & \nhd & \tpr & \auc & Time (s)\\
  \midrule
  \multirow{5}{*}{30} & \underline{\cosmo} &
  \best{0.467 $\pm$ 0.51} & %
  \best{0.962 $\pm$ 0.05} & %
  \rest{0.990 $\pm$ 0.02} & %
  \best{88 $\pm$ 2} \\ %
  & \dagma &
  \rest{0.487 $\pm$ 0.20} & %
  \rest{0.956 $\pm$ 0.03} & %
  \best{0.990 $\pm$ 0.01} & %
  754 $\pm$ 179 \\ %
  & \nocurl &
  0.747 $\pm$ 0.19 & %
  0.938 $\pm$ 0.02 & %
  0.989 $\pm$ 0.00 & %
  826 $\pm$ 32 \\ %
  & \nocurljoint &
  3.107 $\pm$ 0.64 & %
  0.460 $\pm$ 0.06 & %
  0.737 $\pm$ 0.04 & %
  \rest{213 $\pm$ 5} \\ %
  & \notears &
  0.860 $\pm$ 0.76 & %
  0.924 $\pm$ 0.06 & %
  0.975 $\pm$ 0.02 & %
  5199 $\pm$ 130 \\ %
  \midrule
  \multirow{5}{*}{100} & \underline{\cosmo} &
  \rest{0.864 $\pm$ 0.24} & %
  \rest{0.968 $\pm$ 0.01} & %
  \rest{0.992 $\pm$ 0.01} & %
  \best{98 $\pm$ 2} \\ %
  & \dagma &
  \best{0.388 $\pm$ 0.30} & %
  \best{0.975 $\pm$ 0.02} & %
  \best{0.997 $\pm$ 0.00} & %
  422 $\pm$ 103 \\ %
  & \nocurl &
  1.806 $\pm$ 0.40 & %
  0.898 $\pm$ 0.03 & %
  0.982 $\pm$ 0.01 & %
  1676 $\pm$ 31 \\ %
  & \nocurljoint &
  8.756 $\pm$ 2.65 & %
  0.550 $\pm$ 0.05 & %
  0.757 $\pm$ 0.03 & %
  \rest{245 $\pm$ 7} \\ %
  & \notears &
  1.134 $\pm$ 0.81 & %
  0.894 $\pm$ 0.08 & %
  0.989 $\pm$ 0.01 & %
  11618 $\pm$ 1309 \\ %
  \midrule
  \multirow{3}{*}{500} & \underline{\cosmo} &
  \rest{1.426 $\pm$ 0.53} & %
  \best{0.951 $\pm$ 0.03} & %
  \best{0.994 $\pm$ 0.00} & %
  \best{524 $\pm$ 22} \\ %
  & \dagma &
  \best{1.384 $\pm$ 0.38} & %
  \rest{0.849 $\pm$ 0.04} & %
  \rest{0.991 $\pm$ 0.00} & %
  1359 $\pm$ 34 \\ %
  & \nocurljoint &
  8.931 $\pm$ 7.05 & %
  0.430 $\pm$ 0.10 & %
  0.741 $\pm$ 0.08 & %
  \rest{1193 $\pm$ 229} \\ %
  \bottomrule
\end{tabular}
\end{center}

\subsection{ER6 \-- Gaussian Noise}
\begin{center}
\begin{tabular}{clrrrr}
  \toprule
  $d$ & Algorithm & \nhd & \tpr & \auc & Time (s)\\
  \midrule
  \multirow{5}{*}{30} & \underline{\cosmo} &
  4.087 $\pm$ 1.12 & %
  0.838 $\pm$ 0.06 & %
  0.921 $\pm$ 0.04 & %
  \best{89 $\pm$ 4} \\ %
  & \dagma &
  \best{2.367 $\pm$ 0.63} & %
  \rest{0.847 $\pm$ 0.03} & %
  \best{0.958 $\pm$ 0.01} & %
  665 $\pm$ 249 \\ %
  & \nocurl &
  4.480 $\pm$ 0.92 & %
  \best{0.869 $\pm$ 0.03} & %
  0.908 $\pm$ 0.03 & %
  909 $\pm$ 18 \\ %
  & \nocurljoint &
  7.490 $\pm$ 1.18 & %
  0.459 $\pm$ 0.08 & %
  0.672 $\pm$ 0.06 & %
  \rest{226 $\pm$ 6} \\ %
  & \notears &
  \rest{3.327 $\pm$ 1.65} & %
  0.840 $\pm$ 0.07 & %
  \rest{0.922 $\pm$ 0.04} & %
  5239 $\pm$ 427 \\ %
  \midrule
  \multirow{5}{*}{100} & \underline{\cosmo} &
  \rest{9.476 $\pm$ 3.01} & %
  0.771 $\pm$ 0.08 & %
  \rest{0.911 $\pm$ 0.05} & %
  \best{98 $\pm$ 2} \\ %
  & \dagma &
  10.740 $\pm$ 2.83 & %
  0.709 $\pm$ 0.13 & %
  0.902 $\pm$ 0.04 & %
  761 $\pm$ 134 \\ %
  & \nocurl &
  15.044 $\pm$ 1.60 & %
  \rest{0.785 $\pm$ 0.04} & %
  0.888 $\pm$ 0.02 & %
  1687 $\pm$ 26 \\ %
  & \nocurljoint &
  30.719 $\pm$ 5.20 & %
  0.435 $\pm$ 0.03 & %
  0.580 $\pm$ 0.04 & %
  \rest{268 $\pm$ 9} \\ %
  & \notears &
  \best{6.556 $\pm$ 3.10} & %
  \best{0.842 $\pm$ 0.05} & %
  \best{0.944 $\pm$ 0.02} & %
  12053 $\pm$ 940 \\ %
  \midrule
  \multirow{3}{*}{500} & \underline{\cosmo} &
  \rest{25.443 $\pm$ 4.47} & %
  \best{0.736 $\pm$ 0.01} & %
  \best{0.937 $\pm$ 0.01} & %
  \best{526 $\pm$ 100} \\ %
  & \dagma &
  \best{15.952 $\pm$ 1.67} & %
  \rest{0.553 $\pm$ 0.05} & %
  \rest{0.925 $\pm$ 0.01} & %
  3207 $\pm$ 271 \\ %
  & \nocurljoint &
  165.465 $\pm$ 20.86 & %
  0.433 $\pm$ 0.02 & %
  0.558 $\pm$ 0.03 & %
  \rest{1226 $\pm$ 293} \\ %
  \bottomrule
\end{tabular}
\end{center}

\subsection{ER6 \-- Exponential Noise}
\begin{center}
\begin{tabular}{clrrrr}
  \toprule
  $d$ & Algorithm & \nhd & \tpr & \auc & Time (s)\\
  \midrule
  \multirow{5}{*}{30} & \underline{\cosmo} &
  \rest{3.300 $\pm$ 0.95} & %
  \best{0.897 $\pm$ 0.05} & %
  \rest{0.947 $\pm$ 0.03} & %
  \best{89 $\pm$ 2} \\ %
  & \dagma &
  3.480 $\pm$ 1.42 & %
  0.861 $\pm$ 0.06 & %
  0.945 $\pm$ 0.03 & %
  672 $\pm$ 177 \\ %
  & \nocurl &
  4.573 $\pm$ 0.78 & %
  0.846 $\pm$ 0.05 & %
  0.902 $\pm$ 0.03 & %
  897 $\pm$ 13 \\ %
  & \nocurljoint &
  8.700 $\pm$ 0.89 & %
  0.426 $\pm$ 0.07 & %
  0.615 $\pm$ 0.06 & %
  \rest{226 $\pm$ 9} \\ %
  & \notears &
  \best{2.313 $\pm$ 1.55} & %
  \rest{0.881 $\pm$ 0.09} & %
  \best{0.953 $\pm$ 0.04} & %
  5516 $\pm$ 652 \\ %
  \midrule
  \multirow{5}{*}{100} & \underline{\cosmo} &
  10.170 $\pm$ 2.74 & %
  0.768 $\pm$ 0.09 & %
  0.919 $\pm$ 0.04 & %
  \best{99 $\pm$ 3} \\ %
  & \dagma &
  \rest{8.118 $\pm$ 3.10} & %
  \rest{0.793 $\pm$ 0.11} & %
  \rest{0.934 $\pm$ 0.04} & %
  681 $\pm$ 149 \\ %
  & \nocurl &
  14.860 $\pm$ 4.67 & %
  0.685 $\pm$ 0.10 & %
  0.863 $\pm$ 0.06 & %
  1735 $\pm$ 39 \\ %
  & \nocurljoint &
  30.600 $\pm$ 4.34 & %
  0.450 $\pm$ 0.04 & %
  0.591 $\pm$ 0.04 & %
  \rest{267 $\pm$ 8} \\ %
  & \notears &
  \best{5.208 $\pm$ 2.54} & %
  \best{0.796 $\pm$ 0.09} & %
  \best{0.943 $\pm$ 0.03} & %
  12663 $\pm$ 1555 \\ %
  \midrule
  \multirow{3}{*}{500} & \underline{\cosmo} &
  \rest{25.854 $\pm$ 4.28} & %
  \best{0.741 $\pm$ 0.04} & %
  \best{0.943 $\pm$ 0.01} & %
  \best{460 $\pm$ 123} \\ %
  & \dagma &
  \best{16.417 $\pm$ 4.45} & %
  \rest{0.571 $\pm$ 0.11} & %
  \rest{0.925 $\pm$ 0.02} & %
  4069 $\pm$ 580 \\ %
  & \nocurljoint &
  152.336 $\pm$ 31.97 & %
  0.425 $\pm$ 0.02 & %
  0.567 $\pm$ 0.03 & %
  \rest{1363 $\pm$ 306} \\ %
  \bottomrule
\end{tabular}
\end{center}

\subsection{ER6 \-- Gumbel Noise}
\begin{center}
\begin{tabular}{clrrrr}
  \toprule
  $d$ & Algorithm & \nhd & \tpr & \auc & Time (s)\\
  \midrule
  \multirow{5}{*}{30} & \underline{\cosmo} &
  2.840 $\pm$ 1.08 & %
  \best{0.906 $\pm$ 0.04} & %
  \rest{0.954 $\pm$ 0.03} & %
  \best{89 $\pm$ 3} \\ %
  & \dagma &
  \best{2.727 $\pm$ 0.83} & %
  \rest{0.906 $\pm$ 0.02} & %
  \best{0.964 $\pm$ 0.02} & %
  634 $\pm$ 194 \\ %
  & \nocurl &
  5.003 $\pm$ 0.72 & %
  0.811 $\pm$ 0.04 & %
  0.891 $\pm$ 0.03 & %
  902 $\pm$ 9 \\ %
  & \nocurljoint &
  8.153 $\pm$ 0.96 & %
  0.422 $\pm$ 0.07 & %
  0.629 $\pm$ 0.04 & %
  \rest{226 $\pm$ 6} \\ %
  & \notears &
  \rest{2.740 $\pm$ 1.61} & %
  0.791 $\pm$ 0.10 & %
  0.938 $\pm$ 0.04 & %
  5416 $\pm$ 446 \\ %
  \midrule
  \multirow{5}{*}{100} & \underline{\cosmo} &
  10.048 $\pm$ 3.15 & %
  0.780 $\pm$ 0.07 & %
  0.899 $\pm$ 0.06 & %
  \best{100 $\pm$ 3} \\ %
  & \dagma &
  \rest{7.910 $\pm$ 3.05} & %
  \rest{0.805 $\pm$ 0.09} & %
  \rest{0.935 $\pm$ 0.04} & %
  715 $\pm$ 203 \\ %
  & \nocurl &
  11.932 $\pm$ 2.68 & %
  0.742 $\pm$ 0.04 & %
  0.894 $\pm$ 0.03 & %
  1688 $\pm$ 34 \\ %
  & \nocurljoint &
  27.401 $\pm$ 4.42 & %
  0.431 $\pm$ 0.05 & %
  0.600 $\pm$ 0.04 & %
  \rest{266 $\pm$ 4} \\ %
  & \notears &
  \best{4.884 $\pm$ 0.45} & %
  \best{0.833 $\pm$ 0.05} & %
  \best{0.951 $\pm$ 0.01} & %
  12634 $\pm$ 639 \\ %
  \midrule
  \multirow{3}{*}{500} & \underline{\cosmo} &
  \rest{26.148 $\pm$ 4.86} & %
  \best{0.740 $\pm$ 0.04} & %
  \best{0.941 $\pm$ 0.02} & %
  \best{418 $\pm$ 106} \\ %
  & \dagma &
  \best{16.358 $\pm$ 4.94} & %
  \rest{0.563 $\pm$ 0.07} & %
  \rest{0.921 $\pm$ 0.02} & %
  3527 $\pm$ 241 \\ %
  & \nocurljoint &
  125.858 $\pm$ 36.61 & %
  0.367 $\pm$ 0.06 & %
  0.571 $\pm$ 0.02 & %
  \rest{1612 $\pm$ 27} \\ %
  \bottomrule
\end{tabular}
\end{center}

\subsection{SF6 \-- Gaussian Noise}
\begin{center}
\begin{tabular}{clrrrr}
  \toprule
  $d$ & Algorithm & \nhd & \tpr & \auc & Time (s)\\
  \midrule
  \multirow{5}{*}{30} & \underline{\cosmo} &
  1.273 $\pm$ 1.07 & %
  0.907 $\pm$ 0.10 & %
  0.963 $\pm$ 0.06 & %
  \best{89 $\pm$ 2} \\ %
  & \dagma &
  \rest{1.107 $\pm$ 0.37} & %
  \best{0.930 $\pm$ 0.03} & %
  \best{0.985 $\pm$ 0.01} & %
  456 $\pm$ 39 \\ %
  & \nocurl &
  1.573 $\pm$ 0.46 & %
  0.864 $\pm$ 0.04 & %
  0.973 $\pm$ 0.01 & %
  823 $\pm$ 14 \\ %
  & \nocurljoint &
  4.997 $\pm$ 0.98 & %
  0.506 $\pm$ 0.05 & %
  0.732 $\pm$ 0.05 & %
  \rest{226 $\pm$ 8} \\ %
  & \notears &
  \best{0.933 $\pm$ 0.71} & %
  \rest{0.919 $\pm$ 0.05} & %
  \rest{0.984 $\pm$ 0.02} & %
  5313 $\pm$ 184 \\ %
  \midrule
  \multirow{5}{*}{100} & \underline{\cosmo} &
  4.478 $\pm$ 2.88 & %
  0.776 $\pm$ 0.15 & %
  0.874 $\pm$ 0.11 & %
  \best{99 $\pm$ 2} \\ %
  & \dagma &
  \rest{2.024 $\pm$ 0.71} & %
  \rest{0.914 $\pm$ 0.02} & %
  \rest{0.987 $\pm$ 0.00} & %
  396 $\pm$ 53 \\ %
  & \nocurl &
  2.824 $\pm$ 0.39 & %
  0.818 $\pm$ 0.02 & %
  0.980 $\pm$ 0.00 & %
  1679 $\pm$ 27 \\ %
  & \nocurljoint &
  10.556 $\pm$ 6.00 & %
  0.542 $\pm$ 0.07 & %
  0.751 $\pm$ 0.08 & %
  \rest{266 $\pm$ 5} \\ %
  & \notears &
  \best{1.412 $\pm$ 0.59} & %
  \best{0.939 $\pm$ 0.03} & %
  \best{0.990 $\pm$ 0.01} & %
  11156 $\pm$ 170 \\ %
  \midrule
  \multirow{3}{*}{500} & \underline{\cosmo} &
  \rest{4.670 $\pm$ 1.99} & %
  \best{0.912 $\pm$ 0.02} & %
  \best{0.984 $\pm$ 0.00} & %
  \best{460 $\pm$ 70} \\ %
  & \dagma &
  \best{3.825 $\pm$ 0.19} & %
  \rest{0.746 $\pm$ 0.03} & %
  \rest{0.982 $\pm$ 0.00} & %
  1418 $\pm$ 54 \\ %
  & \nocurljoint &
  19.793 $\pm$ 11.03 & %
  0.368 $\pm$ 0.04 & %
  0.698 $\pm$ 0.04 & %
  \rest{1137 $\pm$ 231} \\ %
  \bottomrule
\end{tabular}
\end{center}

\subsection{SF6 \-- Exponential Noise}
\begin{center}
\begin{tabular}{clrrrr}
  \toprule
  $d$ & Algorithm & \nhd & \tpr & \auc & Time (s)\\
  \midrule
  \multirow{5}{*}{30} & \underline{\cosmo} &
  1.393 $\pm$ 1.24 & %
  0.926 $\pm$ 0.05 & %
  0.975 $\pm$ 0.03 & %
  \best{88 $\pm$ 1} \\ %
  & \dagma &
  \rest{1.147 $\pm$ 0.48} & %
  \rest{0.943 $\pm$ 0.03} & %
  \rest{0.982 $\pm$ 0.01} & %
  578 $\pm$ 173 \\ %
  & \nocurl &
  1.987 $\pm$ 0.54 & %
  0.757 $\pm$ 0.08 & %
  0.967 $\pm$ 0.01 & %
  820 $\pm$ 8 \\ %
  & \nocurljoint &
  4.787 $\pm$ 0.99 & %
  0.534 $\pm$ 0.07 & %
  0.761 $\pm$ 0.06 & %
  \rest{227 $\pm$ 7} \\ %
  & \notears &
  \best{0.753 $\pm$ 0.49} & %
  \best{0.943 $\pm$ 0.04} & %
  \best{0.986 $\pm$ 0.01} & %
  5312 $\pm$ 258 \\ %
  \midrule
  \multirow{5}{*}{100} & \underline{\cosmo} &
  3.836 $\pm$ 2.75 & %
  0.864 $\pm$ 0.09 & %
  0.944 $\pm$ 0.05 & %
  \best{98 $\pm$ 2} \\ %
  & \dagma &
  \best{1.532 $\pm$ 0.61} & %
  0.887 $\pm$ 0.04 & %
  \best{0.988 $\pm$ 0.00} & %
  373 $\pm$ 88 \\ %
  & \nocurl &
  2.890 $\pm$ 0.61 & %
  \rest{0.910 $\pm$ 0.02} & %
  0.977 $\pm$ 0.00 & %
  1692 $\pm$ 28 \\ %
  & \nocurljoint &
  6.607 $\pm$ 1.05 & %
  0.474 $\pm$ 0.06 & %
  0.760 $\pm$ 0.06 & %
  \rest{266 $\pm$ 2} \\ %
  & \notears &
  \rest{1.784 $\pm$ 0.52} & %
  \best{0.939 $\pm$ 0.02} & %
  \rest{0.988 $\pm$ 0.00} & %
  11369 $\pm$ 519 \\ %
  \midrule
  \multirow{3}{*}{500} & \underline{\cosmo} &
  \best{3.144 $\pm$ 0.47} & %
  \best{0.919 $\pm$ 0.02} & %
  \best{0.989 $\pm$ 0.00} & %
  \best{457 $\pm$ 81} \\ %
  & \dagma &
  \rest{3.854 $\pm$ 0.34} & %
  \rest{0.750 $\pm$ 0.01} & %
  \rest{0.977 $\pm$ 0.01} & %
  \rest{1384 $\pm$ 33} \\ %
  & \nocurljoint &
  13.763 $\pm$ 8.79 & %
  0.389 $\pm$ 0.05 & %
  0.728 $\pm$ 0.06 & %
  1436 $\pm$ 230 \\ %
  \bottomrule
\end{tabular}
\end{center}

\subsection{SF6 \-- Gumbel Noise}
\begin{center}
\begin{tabular}{clrrrr}
  \toprule
  $d$ & Algorithm & \nhd & \tpr & \auc & Time (s)\\
  \midrule
  \multirow{5}{*}{30} & \underline{\cosmo} &
  \best{1.047 $\pm$ 0.42} & %
  \best{0.938 $\pm$ 0.03} & %
  \best{0.984 $\pm$ 0.01} & %
  \best{88 $\pm$ 1} \\ %
  & \dagma &
  1.347 $\pm$ 0.63 & %
  \rest{0.933 $\pm$ 0.02} & %
  \rest{0.981 $\pm$ 0.01} & %
  528 $\pm$ 67 \\ %
  & \nocurl &
  1.787 $\pm$ 0.52 & %
  0.898 $\pm$ 0.02 & %
  0.969 $\pm$ 0.01 & %
  822 $\pm$ 29 \\ %
  & \nocurljoint &
  5.577 $\pm$ 0.43 & %
  0.549 $\pm$ 0.06 & %
  0.733 $\pm$ 0.05 & %
  \rest{225 $\pm$ 4} \\ %
  & \notears &
  \rest{1.053 $\pm$ 0.59} & %
  0.911 $\pm$ 0.04 & %
  0.978 $\pm$ 0.02 & %
  5429 $\pm$ 251 \\ %
  \midrule
  \multirow{5}{*}{100} & \underline{\cosmo} &
  3.486 $\pm$ 2.62 & %
  0.879 $\pm$ 0.10 & %
  0.947 $\pm$ 0.06 & %
  \best{99 $\pm$ 2} \\ %
  & \dagma &
  \best{1.418 $\pm$ 0.34} & %
  \rest{0.910 $\pm$ 0.03} & %
  \best{0.990 $\pm$ 0.00} & %
  424 $\pm$ 90 \\ %
  & \nocurl &
  3.074 $\pm$ 0.50 & %
  0.893 $\pm$ 0.02 & %
  0.976 $\pm$ 0.00 & %
  1682 $\pm$ 22 \\ %
  & \nocurljoint &
  9.643 $\pm$ 4.59 & %
  0.464 $\pm$ 0.08 & %
  0.712 $\pm$ 0.10 & %
  \rest{267 $\pm$ 9} \\ %
  & \notears &
  \rest{1.586 $\pm$ 1.39} & %
  \best{0.913 $\pm$ 0.06} & %
  \rest{0.987 $\pm$ 0.01} & %
  11820 $\pm$ 985 \\ %
  \midrule
  \multirow{3}{*}{500} & \underline{\cosmo} &
  \best{3.288 $\pm$ 0.50} & %
  \best{0.931 $\pm$ 0.01} & %
  \best{0.992 $\pm$ 0.00} & %
  \best{429 $\pm$ 87} \\ %
  & \dagma &
  \rest{4.055 $\pm$ 0.88} & %
  \rest{0.802 $\pm$ 0.03} & %
  \rest{0.981 $\pm$ 0.00} & %
  1465 $\pm$ 138 \\ %
  & \nocurljoint &
  56.103 $\pm$ 41.06 & %
  0.420 $\pm$ 0.06 & %
  0.648 $\pm$ 0.07 & %
  \rest{1201 $\pm$ 253} \\ %
  \bottomrule
\end{tabular}
\end{center}

\subsection{ER4 \-- Non-linear MLP}
\begin{center}
\begin{tabular}{clrrrr}
  \toprule
  $d$ & Algorithm & \nhd & \tpr & \auc & Time (s)\\
  \midrule
  \multirow{3}{*}{20} & \underline{\cosmo} &
  2.530 $\pm$ 0.45 & %
  0.752 $\pm$ 0.06 & %
  0.923 $\pm$ 0.03 & %
  154 $\pm$ 2 \\ %
  & \model{dagma-mlp} &
  2.853 $\pm$ 0.37 & %
  0.810 $\pm$ 0.15 & %
  0.932 $\pm$ 0.04 & %
  2252 $\pm$ 40 \\ %
  & \model{notears-mlp} &
  3.270 $\pm$ 0.52 & %
  0.904 $\pm$ 0.07 & %
  0.956 $\pm$ 0.02 & %
  3554 $\pm$ 154 \\ %
  \midrule
  \multirow{2}{*}{40} & \underline{\cosmo} &
  3.295 $\pm$ 0.45 & %
  0.712 $\pm$ 0.05 & %
  0.925 $\pm$ 0.02 & %
  177 $\pm$ 2 \\ %
  & \model{dagma-mlp} &
  3.633 $\pm$ 0.81 & %
  0.804 $\pm$ 0.10 & %
  0.942 $\pm$ 0.03 & %
  2622 $\pm$ 28 \\ %
  \midrule
  \multirow{2}{*}{100} & \underline{\cosmo} &
  5.164 $\pm$ 1.38 & %
  0.582 $\pm$ 0.04 & %
  0.895 $\pm$ 0.02 & %
  311 $\pm$ 10 \\ %
  & \model{dagma-mlp} &
  2.337 $\pm$ 0.18 & %
  0.506 $\pm$ 0.05 & %
  0.912 $\pm$ 0.01 & %
  4866 $\pm$ 69 \\ %
  \bottomrule
\end{tabular}
\end{center}

\subsection{SF4 \-- Non-linear MLP}
\begin{center}
\begin{tabular}{clrrrr}
  \toprule
  $d$ & Algorithm & \nhd & \tpr & \auc & Time (s)\\
  \midrule
  \multirow{3}{*}{20} & \underline{\cosmo} &
  1.525 $\pm$ 0.18 & %
  0.743 $\pm$ 0.05 & %
  0.954 $\pm$ 0.03 & %
  155 $\pm$ 3 \\ %
  & \model{dagma-mlp} &
  1.408 $\pm$ 0.26 & %
  0.746 $\pm$ 0.11 & %
  0.968 $\pm$ 0.01 & %
  2270 $\pm$ 27 \\ %
  & \model{notears-mlp} &
  1.040 $\pm$ 0.33 & %
  0.922 $\pm$ 0.08 & %
  0.980 $\pm$ 0.02 & %
  3400 $\pm$ 184 \\ %
  \midrule
  \multirow{2}{*}{40} & \underline{\cosmo} &
  2.271 $\pm$ 0.24 & %
  0.583 $\pm$ 0.05 & %
  0.953 $\pm$ 0.01 & %
  174 $\pm$ 5 \\ %
  & \model{dagma-mlp} &
  1.875 $\pm$ 0.34 & %
  0.744 $\pm$ 0.10 & %
  0.972 $\pm$ 0.01 & %
  2588 $\pm$ 45 \\ %
  \midrule
  \multirow{2}{*}{100} & \underline{\cosmo} &
  3.144 $\pm$ 0.23 & %
  0.455 $\pm$ 0.04 & %
  0.945 $\pm$ 0.01 & %
  313 $\pm$ 6 \\ %
  & \model{dagma-mlp} &
  3.051 $\pm$ 0.14 & %
  0.260 $\pm$ 0.04 & %
  0.961 $\pm$ 0.01 & %
  4883 $\pm$ 77 \\ %
  \bottomrule
\end{tabular}
\end{center}

\subsection{ER6 \-- Non-linear MLP}
\begin{center}
\begin{tabular}{clrrrr}
  \toprule
  $d$ & Algorithm & \nhd & \tpr & \auc & Time (s)\\
  \midrule
  \multirow{3}{*}{20} & \underline{\cosmo} &
  2.995 $\pm$ 0.45 & %
  0.667 $\pm$ 0.07 & %
  0.919 $\pm$ 0.03 & %
  155 $\pm$ 2 \\ %
  & \model{dagma-mlp} &
  2.992 $\pm$ 0.39 & %
  0.712 $\pm$ 0.08 & %
  0.917 $\pm$ 0.02 & %
  2252 $\pm$ 35 \\ %
  & \model{notears-mlp} &
  2.895 $\pm$ 0.47 & %
  0.862 $\pm$ 0.08 & %
  0.949 $\pm$ 0.02 & %
  3557 $\pm$ 130 \\ %
  \midrule
  \multirow{2}{*}{40} & \underline{\cosmo} &
  4.837 $\pm$ 0.66 & %
  0.598 $\pm$ 0.07 & %
  0.891 $\pm$ 0.02 & %
  174 $\pm$ 3 \\ %
  & \model{dagma-mlp} &
  3.732 $\pm$ 0.71 & %
  0.633 $\pm$ 0.13 & %
  0.919 $\pm$ 0.02 & %
  2570 $\pm$ 44 \\ %
  \midrule
  \multirow{2}{*}{100} & \underline{\cosmo} &
  6.049 $\pm$ 0.91 & %
  0.501 $\pm$ 0.04 & %
  0.875 $\pm$ 0.02 & %
  308 $\pm$ 4 \\ %
  & \model{dagma-mlp} &
  3.790 $\pm$ 0.24 & %
  0.535 $\pm$ 0.05 & %
  0.908 $\pm$ 0.01 & %
  4807 $\pm$ 80 \\ %
  \bottomrule
\end{tabular}
\end{center}

\subsection{SF6 \-- Non-linear MLP}
\begin{center}
\begin{tabular}{clrrrr}
  \toprule
  $d$ & Algorithm & \nhd & \tpr & \auc & Time (s)\\
  \midrule
  \multirow{3}{*}{20} & \underline{\cosmo} &
  1.855 $\pm$ 0.28 & %
  0.748 $\pm$ 0.06 & %
  0.959 $\pm$ 0.01 & %
  154 $\pm$ 2 \\ %
  & \model{dagma-mlp} &
  1.833 $\pm$ 0.52 & %
  0.766 $\pm$ 0.13 & %
  0.957 $\pm$ 0.03 & %
  2249 $\pm$ 40 \\ %
  & \model{notears-mlp} &
  1.355 $\pm$ 0.39 & %
  0.926 $\pm$ 0.08 & %
  0.980 $\pm$ 0.01 & %
  3512 $\pm$ 193 \\ %
  \midrule
  \multirow{2}{*}{40} & \underline{\cosmo} &
  3.394 $\pm$ 0.25 & %
  0.501 $\pm$ 0.06 & %
  0.946 $\pm$ 0.01 & %
  174 $\pm$ 4 \\ %
  & \model{dagma-mlp} &
  2.720 $\pm$ 0.38 & %
  0.724 $\pm$ 0.08 & %
  0.967 $\pm$ 0.01 & %
  2572 $\pm$ 100 \\ %
  \midrule
  \multirow{2}{*}{100} & \underline{\cosmo} &
  4.744 $\pm$ 0.21 & %
  0.342 $\pm$ 0.04 & %
  0.931 $\pm$ 0.01 & %
  308 $\pm$ 6 \\ %
  & \model{dagma-mlp} &
  4.149 $\pm$ 0.17 & %
  0.391 $\pm$ 0.06 & %
  0.968 $\pm$ 0.01 & %
  4699 $\pm$ 338 \\ %
  \bottomrule
\end{tabular}
\end{center}

\end{document}